%% file: main.tex
\documentclass[12pt]{article}

\usepackage{hyperref}
\hypersetup{
  colorlinks=true,       
  linkcolor=red,         
  hidelinks,             
  hyperfootnotes=false   
}
\usepackage{amsmath,amsfonts,amsthm,mathrsfs,amsopn,amssymb,xcolor,mathtools}
\usepackage{longtable} 
\usepackage{lineno} 
\usepackage{soul} 
\usepackage{authblk}
\usepackage{graphicx}
\usepackage{booktabs}
\usepackage{placeins}
\newcommand{\blfootnote}[1]{\begingroup\renewcommand{\thefootnote}{}\footnote{#1}\addtocounter{footnote}{-1}\endgroup}

\allowdisplaybreaks[4]

\input{shared}

\title{Understanding Diffusion Models via Ratio-Based Function Approximation with SignReLU Networks}

\author{SUN Luwei$^{1}$,
SHEN Dongrui$^{1}$,
FENG Chuanwen$^{1}$,
LI Jianfei$^{2}$,
ZHAO Yulong$^{1}$,
FENG Han$^{1}$
}

\date{\vspace*{-1cm}}

\begin{document}
\maketitle

\blfootnote{$^1$SUN Luwei, SHEN Dongrui, ZHAO Yulong, and FENG Han are with the Department of Mathematics, City University of Hong Kong, Kowloon, Hong Kong (email: luweisun2-c@my.cityu.edu.hk; dongrshen2-c@my.cityu.edu.hk; chuanfeng7@cityu.edu.hk; zhaoyulong@gmail.com; hanfeng@cityu.edu.hk).}
\blfootnote{$^2$LI Jianfei is with the Ludwig-Maximilians-Universit{\"a}t M{\"u}nchen, Munich, Germany (email: lijianfei@math.lmu.de).}

\begin{abstract}
    Motivated by the core challenges in conditional generative modeling, where the target conditional density inherently takes the form of a ratio $\frac{f_1}{f_2}$,  this paper develops a rigorous theoretical framework for approximating such ratio-type functionals. Here, $f_1$ and $f_2$ are kernel-induced marginal densities to capture structured interactions, a setting central to diffusion-based generative models. We present a concise and clear proof for the approximation of these ratio-type functionals by deep neural networks with the SignReLU activation function, leveraging the activations special piecewise structure. Under standard regularity assumptions, we establish $L^p(\Omega)$ approximation bounds and the convergence rates for ratio-type targets. Specializing to Denoising Diffusion Probabilistic Models (DDPMs), we
represent the population-optimal reverse predictor as a ratio functional,
construct a corresponding SignReLU approximation, and derive explicit
bounds on the excess Kullback--Leibler (KL) risk between the generated
distribution and the true data distribution. Our analysis decomposes this excess risk into interpretable approximation and estimation error components. The results deliver rigorous generalization guarantees for the finite-sample training of diffusion based generative models.

\end{abstract}

\section{Introduction}
Conditional density estimation is a central component of modern generative modeling, enabling models to learn mappings from side information to complex data distributions. A rich spectrum of conditional generation methods—achieving state-of-the-art results in image synthesis~\cite{song2020score,dhariwal2021diffusion}, video generation~\cite{ho2022video}, and audio processing~\cite{chen2020wavegrad,kong2020diffwave}—learn a parametric conditional law $p_\theta(y \mid x)$ that maps side information $x \in {X}$ to a distribution over outputs $y \in {Y}$, where $\theta$ denotes model parameters.  Formally, the statistical target can be related to the conditional density $f_{Y \mid X}(y \mid x)=\frac{f_{X, Y}(x, y)}{ f_X(x)}$, which is itself a ratio of densities or probability mass functions. This ratio-centric perspective connects our study to broader density-ratio methodologies for probabilistic prediction and likelihood-free inference, where a quotient $\frac{f_1}{f_2}$ is approximated directly under weak modeling assumptions. 

In practice, $\frac{f_1}{f_2}$-type quantities are universal across density-based generative paradigms. Diffusion models minimize denoising score-matching objectives\cite{ho2020denoising,song2020score}, normalizing flows maximize exact likelihood \cite{dinh2016density}, conditional VAEs optimize a conditional evidence lower bound \cite{sohn2015learning} , and conditional GANs learn implicit distributions via adversarial training\cite{mirza2014conditionalgenerativeadversarialnets}. Despite their algorithmic differences, these methods often hinge on estimating objects that can be viewed as density ratios. 

Unlike standard targets in density estimation—such as marginal densities or regression maps—that approximate a single, standalone function, $\frac{f_1}{f_2}$–type functionals quantify relative probabilistic dependence between two objects. Their importance is further amplified when the denominator admits an integral form $f_2 = \int_\Omega \Phi(x,y) g(y) dy$, where $\Phi(x,y)$ is an interaction kernel (e.g., Gaussian, dot-product) capturing similarity between $x$ and $y$, and $g(y)$ is a weight function (e.g.,a marginal or prior density). Such kernel-induced integrals are ubiquitous whenever marginalization over latent or auxiliary variables is required and closed-form expressions are unavailable.

Deep neural networks (DNNs) are widely used across various tasks in generative modeling. A fully connected DNN with input vector $\mathbf{x} = (x_1, x_2, \ldots, x_d)^\top \in \mathbb{R}^d$ and depth $\mathcal{L} - 1$ is defined as the composition:
\begin{equation}
\label{eq:pre1}
    \begin{gathered}
        \NN(\CW,\CL,\mathcal{M}) :=\{\CA_{\CL} \circ \sigma \circ \CA_{{\CL-1}} \circ \sigma \circ \cdots \circ \sigma \circ \CA_{1} \circ \sigma \circ \CA_{0}: \\ 
        \Vert{(A_{\CL}},{b_{\CL})}\Vert_{\infty}\prod_{\ell=0}^{\CL-1}\max\{\Vert{(A_{\ell}},{b_{\ell})}\Vert_{\infty},1\}\leq \mathcal{M}\},
    \end{gathered}    
\end{equation}
where each affine transformation $\mathcal{A}_i(\mathbf{x}) := \boldsymbol{A}_i \mathbf{x} + \boldsymbol{b}_i$ is parameterized by a weight matrix $\boldsymbol{A}_i \in \mathbb{R}^{d_i \times d_{i-1}}$ and a bias vector $\boldsymbol{b}_i \in \mathbb{R}^{d_i}$. The activation function $\sigma$ is applied element-wise. Each composition $\sigma \circ \mathcal{A}_i$ constitutes the $i$-th hidden layer, which has a width of $d_i$. The final affine transformation, $\mathcal{A}_L$, is the output layer. A network's width is at most $\mathcal{W}$ if $\max _{1 \leq i \leq L} d_i \leq \mathcal{W}$. The quantity $\mathcal{M}$ is a bound on the product of the layerwise
parameter norms.

Despite significant advances in the theoretical understanding of the approximation capabilities of DNNs, approximating functionals of the form $\frac{f_1}{f_2}$ remains challenging for several reasons. First, the functions often lack closed forms and may exhibit sharp variations or blow-up behavior as $f_2$ approaches zero, necessitating careful control of denominators and regularity near low-density regions. Second, the ratio must be accurate across the entire support to ensure reliable probabilistic inference, not merely in high-density subsets. Third, when $f_2$ involves kernel integrals, approximation errors accumulate through both the numerator and denominator, compounding numerical instability and statistical error. Motivated by these considerations, we investigate neural network-based estimators that approximate functions of the form $\frac{f_1}{f_2}$, even in discrete input spaces and without closed-form expressions, thus offering a flexible, data-driven route to conditional generation and likelihood-free inference.

\paragraph{Related work and background}
A substantial literature studies neural-network approximation under different
activation functions. Classical results based on Maurey's empirical method
\cite{256500,DeVore_1998,a75dd10b-6c44-397d-83bb-a49069fd5311} establish
approximation rates for functions in convex hulls of bounded dictionaries.
Barron's seminal work \cite{256500} derives dimension-dependent
\(L_2(\mathbb{R}^d)\) approximation rates under a Fourier-moment condition.
Suzuki \cite{suzuki2018adaptivity} studies deep ReLU approximation over Besov
spaces, while Siegel et al.\ \cite{siegel2024sharp} establish sharper nonlinear
approximation rates for shallow ridge-function networks.

Theoretical studies of diffusion models have analyzed sampling convergence under
approximate scores, numerical discretization of the reverse process, and neural
approximation and finite-sample estimation of score functions
\cite{lee2022convergence,chen2023score,chen2023samplingeasylearningscore}. These works primarily
study score-based parameterizations and propagate score and discretization errors
to distributional metrics. In contrast, we directly analyze the
population-optimal DDPM reverse predictor through its conditional-mean ratio
representation. We establish its SignReLU approximation and finite-sample
estimation bounds and then propagate the resulting prediction error to the excess
KL risk of the generated distribution.

\paragraph{Our approach and scope}
We focus on neural approximation of rational functionals $\frac{f_1}{f_2}$ whose denominator may be a kernel-induced integral. Our construction combines neural approximations of the numerator and
denominator with a SignReLU division module, enabling explicit control of
the resulting approximation error and network complexity.

We then specialize this construction to diffusion models by expressing the
population-optimal DDPM reverse predictor as a ratio functional. Using
stochastic-calculus and empirical-process arguments, we derive approximation
and finite-sample estimation bounds and propagate them to the excess KL risk
of the generated distribution.

\paragraph{Our contributions}
\begin{itemize}
    \item Foundational approximation bounds. We establish approximation guarantees for a general class of ratio functionals $\frac{f_1}{f_2}$ in which ${f_1}$ and ${f_2}$ arise from kernel-induced marginal densities. The ratio is approximated by well constructed DNNs with explicit control of approximation error in both numerator and denominator, with explicit control of the denominator-dependent approximation error. Specialization to diffusion models. We express the
population-optimal DDPM reverse predictor as a ratio functional and
construct a corresponding SignReLU approximation. We derive localized
approximation and finite-sample estimation bounds and propagate these
errors to an excess-KL bound for the generated distribution.
    
\end{itemize}

\section{Analysis on "$\frac{f_1}{f_2}$ - type" functions}
In this section, we analyze the approximation rate for a broad class of ratio functionals of the form $\frac{f_1}{f_2}$, where both $f_1$ and $f_2$ are derived from kernel-induced marginal densities. Accordingly, we assume that $f_1, f_2 \in \mathcal{S}$, where $\mathcal{S} \subset L^p(\Omega)$ for some $2 \leq p < \infty$, defined as:
\begin{align}
\label{eq:set}
	\mathcal{S}=\left\{f(x)=\int_{\Omega} \Phi(x, y) g(y) d y: g \in L^1(\Omega), \Phi(x, y)=\sum_{j=1}^m \phi_j(x^T A_j y),\left\|\phi_j^{\prime \prime}\right\|_{\infty}<\infty\right\}.
\end{align}
We consider random variables $\mathbf{X}$ and $\mathbf{Y}$ supported on a bounded domain $\Omega \subset [-1,1]^d$. In this setting, $\Phi(x, y)$ denotes an interaction kernel between $x$ and $y$, and the function $f(x)$ can be interpreted as a marginal density induced by the joint structure of $\Phi(x, y)$ and a weighting function $g(y)$. The kernel is defined as $\Phi(x,y)=\sum_{j=1}^m \phi_j(x^T A_j y)$ where each $\phi_j$ is a univariate function and $A_j$ is a learnable linear transformation. This structure is specifically designed to extract meaningful interaction patterns between $x$ and $y$, while enabling accurate approximation of their joint density $f_{X,Y}(x, y)$—a fundamental requirement for generative modeling and conditional density estimation. The model's capacity to capture rich interactions arises from the directional similarity metric $x^\top A_j y$, which quantifies alignment between $x$ and linearly transformed versions of $y$, and from the flexibility of the learnable matrices $A_j$, which serve as feature extractors. By projecting $y$ into informative subspaces, the model highlights salient interactions with $x$. Moreover, the summation over univariate functions $\phi_j(\cdot)$ allows the kernel to represent complex, nonlinear dependencies and approximate intricate joint distributions effectively.

While ReLU and SiLU are widely used in practical neural architectures, the present analysis requires explicit control of a ratio-type target. Li et al.~\cite{li2024signrelu} show that SignReLU networks admit compact constructions for polynomial and rational operations, including a division module on domains where the denominator is bounded away from zero. This structural property makes SignReLU particularly convenient for deriving explicit approximation guarantees for the ratio functionals considered here.

In this work, we study the approximation of rational functions formed by $f_1, f_2 \in \mathcal{S}$ using neural networks $\mathcal{N}\mathcal{N}(\mathcal{W}, \mathcal{L}, \mathcal{M})$, as defined in Eq.\ref{eq:pre1}, equipped with the SignReLU activation function:
\begin{align}
\label{eq:signrelu}
\operatorname{SignReLU}(x; \alpha) = 
\begin{cases}
x, & \text{if } x > 0, \\
\alpha \dfrac{x}{1 + |x|}, & \text{if } x \leq 0,
\end{cases}
\quad \text{where } \alpha \text{ is a tunable parameter.}
\end{align}
The rational negative branch of SignReLU enables an explicit realization of
division on bounded domains. This property is well suited to the ratio-type
targets considered here and allows the corresponding approximation error and
network complexity to be characterized explicitly. Throughout the following analysis, we denote by $C(a, b)$ a constant that depends only on $a$ and $b$. We begin by presenting our approximation results for rational functions.
\begin{theorem}
\label{thm:f1/f2}
    Let $f_1$ and $f_2 \in \mathcal{S}$. Assume that there exist constants
$0<c<C<\infty$ such that
\(
c\leq f_2(x)\leq C\) and \(
|f_1(x)|\leq C\) for all $x\in [-1,1]^d$. Then there exists $\phi$ realized by a SignReLU neural network $\CN\CN(\CW,\CL,\mathcal{M})$ with depth $\CL = 7$, width $\CW = O(n+9)$ and parameter norm $\mathcal{M}>0$, such that
    \begin{align*}
    \Bigl|\frac{f_1}{f_2}-\phi\Bigr|
\precsim \mathcal{M}^{-1}n^{-\frac{1}{2}-\frac{3}{2d}}.
   \end{align*}
\end{theorem}

The proof of Theorem \ref{thm:f1/f2} relies on the following proposition and lemma. The Proposition \ref{prop:kernel} shows a key property of the function class $\mathcal{S}$ that its elements can be efficiently approximated by shallow neural networks equipped with SignReLU activation. The Lemma \ref{lemma:f1/f2sign} proved in \cite{li2024signrelu} shows the implementation of the division gate by a SignReLU Network.
\begin{prop}
	\label{prop:kernel}
Let $2 \leq p<\infty$. For any integer $n>1$ and  $f \in \mathcal{S}$, there exists $f_n\in \CN\CN(\CW,1,\mathcal{M})$ with $\CW=O(n)$ and $\mathcal{M}>0$ such that 
\begin{align*}
    \left|f-f_n\right| \precsim n^{-\frac{1}{2}-\frac{3}{2d}}.
\end{align*}
\end{prop} The detailed proof is postponed in Appendix~\ref{sec:proof of kernel}.
\begin{lemma}(SignReLU neural networks for division gates\cite{li2024signrelu})
\label{lemma:f1/f2sign}
 Let $0<c<C$ be constants. There exists a mapping $\psi$ realized by a SignReLU neural network with depth $\mathcal{L}=6$, width $\mathcal{W} \leq 9$ and number of parameters $\mathcal{N} \leq 71$ such that
$$
\psi(x, y)=\frac{y}{x}, \quad \forall x \in[c, C], y \in[-C, C].
$$
\end{lemma}

{\bf Sketch Proof of Theorem \ref{thm:f1/f2}:} For $f_1,f_2\in S$, by Proposition~\ref{prop:kernel}, there exist shallow SignReLU networks $\tilde f_1, \tilde f_2\in \mathcal{N}\mathcal{N}(\mathcal{W}, 1, \mathcal{M})$ to approximate $f_1, f_2$, respectively. Then the composition $\psi(\tilde f_1, \tilde f_2)\in \mathcal{N}\mathcal{N}(\mathcal{W}+9, 7, \mathcal{M})$ will be used to approximate the ratio-type target $f_1/f_2$. Here, the depth $\mathcal{L}=7$ results from composing the
one-hidden-layer approximations of $f_1$ and $f_2$ with the depth-$6$
division module. It is a sufficient depth for the stated approximation
bound, rather than a necessary or minimal architectural requirement.
For more details of the proof, see Appendix \ref{sec:proof of thm1}. 

 \section{Error Analysis of diffusion model}
In this section, we apply our approximation strategy for ratio-type functions $\frac{f_1}{f_2}$ to explicitly construct a SignReLU architecture that approximates diffusion model objectives. Taking into account practical training procedures, we derive finite-sample estimation error bounds using tools from stochastic calculus and statistical learning theory. The excess Kullback–Leibler (KL) risk of the diffusion model is then characterized by a trade-off between approximation and estimation errors.
\subsection{Problem Formulation}
While pioneering frameworks of diffusion models include denoising score matching with Langevin dynamics (SMLD)\cite{song2019generative} and denoising diffusion probabilistic models (DDPMs)\cite{ho2020denoising}, due to their mathematical equivalence to SMLD under specific re-parameterizations\cite{lee2022convergence}, our analysis focuses on DDPMs. 

DDPMs construct two discrete Markov chains:

(1) a forward process
$$
\mathbf{X} \rightarrow \mathbf{Z}_1 \rightarrow \cdots \rightarrow \mathbf{Z}_T,
$$
which starts from a sample drawn from the target data distribution (e.g., natural images) and progressively perturbs it until it approaches a noise distribution (e.g., a standard Gaussian).

(2) a reverse process
$$
\mathbf{\hat Z}_T \rightarrow \mathbf{\hat Z}_{T-1} \rightarrow \cdots \rightarrow \mathbf{\hat Z}_0,
$$
which begins from pure noise (e.g., a standard Gaussian) and iteratively transforms it into samples whose distribution approximates the target data distribution.
\paragraph{The Forward Process.} The forward process maps $\mathbf{X}$ through a series of intermediate variables $\{\mathbf{Z}_1, \mathbf{Z}_2, \ldots, \mathbf{Z}_T\}$ as 
\begin{align*}
    &\mathbf{X} \sim q_0, \\
    &\mathbf{Z}_1  =\sqrt{1-\beta_1} \mathbf{X}+\sqrt{\beta_1} \mathbf{W}_1,\\
&\mathbf{Z}_t  =\sqrt{1-\beta_t} \mathbf{Z}_{t-1}+\sqrt{\beta_t} \mathbf{W}_t, \quad 2 \leq t \leq T,
\end{align*} 
where $\left\{\mathbf{W}_t\right\}_{1 \leq t \leq T}$ indicates a sequence of independent noise vectors drawn from $\mathbf{W}_t \stackrel{\text { i.i.d. }}{\sim} \mathcal{N}\left(0, \mathbf{I}_d\right)$. The hyper-parameters $\left\{\beta_t \in(0,1)\right\}$ represent prescribed learning rate schedules that control the variance of the noise injected in each step. Let $\alpha_t \coloneq \prod_{s=1}^{t}(1-\beta_s)$, then it can be straightforwardly verified that for every $1 \leq t \leq T$,
\begin{align*}
    \mathbf{Z}_t=\sqrt{\alpha_t} \mathbf{X}+\sqrt{1-\alpha_t} \epsilon_t \quad \text { for some } \epsilon_t \sim \mathcal{N}\left(0, \mathbf{I}_d\right). 
\end{align*} 
Let $q(z_t \mid z_{t-1})$ denote the conditional distribution in the forward diffusion process. The marginal posterior $q(z_t \mid x)$ is given by
\begin{align*}
    q(z_t \mid x)\sim \mathcal{N}(\sqrt{\alpha_t}x,\sigma_{q,t}^2\mathbf{I}_d) \quad \text{where } \sigma_{q,t}=\sqrt{1-\alpha_t}.
\end{align*}
Given an i.i.d. dataset $\{x_i\}_{i=1}^m$ sampled from the target distribution $q_0$, we first generate paired samples $\{z_{t,i}\}$ for $i=1,\ldots, m$ and $t=1,\ldots, T$, by propagating each $x_i$ through the diffusion forward process. These $\{z_{t,i}\}$ samples are then used to train the neural network that models the reverse process, enabling the network to iteratively denoise latent variables and ultimately generate samples consistent with the target distribution $q_0$.

\paragraph{The Backward Process.}
The reverse conditionals $q(z_{t-1}\mid z_{t})$ are approximated stepwise by a decoder, which is constructed by a sequence of neural networks $\NN(\CW,\CL,\mathcal{M})$.

Precisely, for each $\phi_t \in \NN(\CW,\CL,\mathcal{M})$, t=1, \ldots, T, let
\begin{align*}
    p\left(z_{t-1} \mid z_t\right)=\left(2 \pi  \sigma_{p,t}^2\right)^{-\frac{d}{2}} \exp (-\frac{\| z_{t-1}-\phi_t( z_t) \|_2^2}{2  \sigma_{p,t}^2})\sim \CN_{z_{t-1}}\left(\phi_t( z_t), \sigma_{p,t}^2 \mathbf{I}_d\right),
 \end{align*}
and $p_0$ denote the distribution induced at $t=0$ by learned transitions,

\begin{equation}
\label{p0}      
{p}_0(x):= {p}_0(\{\phi_t\};x)=\int q_T\left({ z}_T\right) \cdot p\left(x \mid { z}_1\right) \cdot \prod_{t=2}^T p\left({ z}_{t-1} \mid { z}_t\right) d{ z}_1, \ldots, {z}_T,
\end{equation}
   which will be used to approximate $q_0$, the target distribution. 
  
The specific selection of $\phi_t$'s is based on the minimization of the expected log-likelihood $\mathbb{E}_{q_0}[\log p_0(x)]$. Practically, given $m$ i.i.d. samples $\{x_i\}_{i=1}^m$ drawn from $q_0$ and $\{z_{t,i}\}$ generated from the forward process, an empirical sequence of optimizers $\{\hat \phi_t\}_{t=1}^{T}$ is defined by
    \begin{align}
          \{\hat \phi_t\}_{t=1}^{T}=&\argmax\limits_{\phi_t\in \NN(\CW,\CL,\mathcal{M})}\frac{1}{m}\sum_{i=1}^m\log p_0(x_i)\notag \label{eq:ddpm1}\\
          =&\argmin _{\phi_t\in \NN(\CW,\CL,\mathcal{M})}\frac{1}{m}\sum_{i=1}^m\left\{\frac{1}{2 \sigma_{p, 1}^2}\left|x_i-\phi_t\left(z_{1,1}\right)\right|^2+\sum_{t=2}^T \frac{1}{2 \sigma_{p, t}^2}\left|z_{t-1, i}-\phi_t\left(z_{t, i}\right)\right|^2\right\}.
            \end{align}
On the other hand, the distribution $q_T(z_T)$ in \eqref{p0} is replaced by the standard Gaussian distribution,
$p_T(\hat{z}_T) \sim \mathcal{N}\left(0, \mathbf{I}_d\right).$
By applying the learned denoisers $\{\hat{\phi}_t\}_{t=1}^T$ from Eq.~\ref{eq:ddpm1}, the reverse process is then conducted as a sequence of transformations
\begin{align*}
\hat{\mathbf{Z}}_{t-1}=\phi_t\left(\hat{\mathbf{Z}}_t\right), \quad \text { for } t=T, T-1, \ldots, 1.
\end{align*}
and the reverse transition at each time step $t$ is modeled as $\hat p\left(\hat z_{t-1} \mid\hat z_t\right)\sim \mathcal{N}_{\hat z_{t-1}}\left(\hat \phi_t(\hat z_t), \sigma_{p,t}^2 \mathbf{I}_d\right).$
Moreover, the distribution of the generated output $\hat{\mathbf{Z}}_0$ is given by
  \begin{align*}
      \hat{p}_0(\hat z_0)=\int p_T\left({\hat z}_T\right) \cdot \hat{p}\left({\hat z}_0 \mid {\hat z}_1\right) \cdot \prod_{t=2}^T \hat{p}\left({\hat z}_{t-1} \mid {\hat z}_t \right) d{\hat z}_1, \ldots, {\hat z}_T,
  \end{align*}
  and the convergence of the DDPM is evaluated by comparing the learned distribution $\hat{p}_0(\hat{z}_0)$ with the target data distribution $q_0(x)$:
  \begin{align*}
    \text { (goal) } \quad  \hat{p}_0(\hat z_0) \stackrel{\mathrm{~d}}{\approx} q_0(x) .
\end{align*}
  \paragraph{Excess KL Risk}  
  The performers of DDPM as a density estimator is assessed via the Kullback–Leibler (KL) divergence. To this end, we introduce a sequence of networks $\{\tilde \phi_t\}_{t=1}^T$ defined as
\begin{align}
\{\tilde \phi_t\}_{t=1}^T = \argmax_{\phi_t \in \mathcal{N}(\mathcal{W}, \mathcal{L}, \mathcal{M})} \int \log p_0(\{\phi_t\}; x) \ dx, \label{eq:ddpm2}
\end{align}
which yields the expected optimized distribution $\tilde{p}_0(x):= p_0(\{\tilde\phi_t\}; x)$. Following the DDPM framework, our analyze about the divergence between the true data distribution and the model-generated distribution will be decomposed as
       \begin{align}
          \label{eq:excess}
      \ D_{KL}\left(q_0(x) \| \hat{p}_0(\hat z_0)\right)&=\underbrace{D_{KL}\left(q_0(x) \| \tilde{p}_0(x)\right)}_{\text{Approximation Error}}\notag\\
      & \quad+ \underbrace{D_{KL}\left(q_0(x) \| \hat{p}_0(\hat z_0)\right)-D_{KL}\left(q_0(x) \|\tilde{p}_0(x)\right).}_{\text{Estimation Error}}
       \end{align}
The estimation error primarily stems from substituting the population expectation with a finite-sample average, while the approximation error originates from two sources: the inherent representational limitations of neural networks, and the use of estimated noise (rather than true noise) to infer the diffusion reverse process.

In the subsequent subsections, we conduct a rigorous error analysis of the DDPM by first establishing bounds for these two error components, then combining them to characterize the overall excess Kullback–Leibler (KL) risk of the generative model. The analysis is conducted under a standard assumption on the target data distribution, which requires its probability density function to be uniformly bounded and supported on a compact domain.
\begin{assumption}[Boundness and compactness of the target distribution]
\label{assump:target1}
The input density $q_0$ is supported on $\Omega = [-1,1]^d$, satisfies $q_0 \in C^k(\Omega)$ for some $k \ge 0$, and there exists a constant $C_{q_0} \in [1,\infty)$ such that $C_{q_0}^{-1} \leq q_0 \leq C_{q_0}$.
\end{assumption}

Under Assumption \ref{assump:target1} and the forward process defined by $\mathbf{Z}_t \mid \mathbf{X} \sim \mathcal{N}\left(\sqrt{\alpha_t} \mathbf{X},\left(1-\alpha_t\right) \mathbf{I}_d\right)$, the following concentration bound holds for any $\xi>0$,
\begin{equation*}
\mathbb{P}\!\left(\|\mathbf{Z}_t\|_2 \leq \sqrt{1-\alpha_t}\,(\sqrt{2\xi}+\sqrt{d}+1)\right) \ge 1-e^{-\xi}.
\end{equation*}
A bounded-output modification is applied to the network when $\left|\mathbf{Z}_t\right|$ is considerably large, as $\left|\mathbf{Z}_t\right|>\sqrt{1-\alpha_t}(\sqrt{2 \xi}+\sqrt{d}+1)$. This modification is justified by the fact that, for sufficiently large $\xi$, the probability of this event is negligible, being on the order of $e^{-\xi}$.

\subsection{Approximation Error}
This section analyzes the approximation error incurred when using a deep neural network to learn the reverse conditional distribution of the forward process. In the forward process, the distribution is given by $q(z_t \mid z_{t-1})$, and the objective is to approximate the corresponding reverse distribution using a neural network. Our analysis of the approximation error proceeds by decomposing it into the following components (see Appendix \ref{sec:Decomposition of approx} for details),
\begin{equation}
    \begin{aligned}
    \label{eq:approx1}
    D_{KL}\left(q_0(x) \| \tilde{p}_0(x)\right) &\asymp D_{KL}\left[q\left(z_T \mid x\right) \| q_T\left(z_T\right)\right] \\
    &\quad +T\underset{\text{step t in T}}{\sup} E_{\mathbf{X}} E_{\mathbf{Z}_t \mid \mathbf{X}} D_{KL}\left[q\left(z_{t-1} \mid z_t, x\right) \| \tilde{p}\left(z_{t-1} \mid z_t\right)\right].
\end{aligned}
\end{equation}
The main result is summarized in the following theorem.
\begin{theorem}
\label{thm:approx}
    Consider the DDPM architecture under Assumption~\ref{assump:target1}, where the target data distribution satisfies $\mathbf{X} \sim q_0$. Let $2 \leq p<\infty$. Then, there exists $\NN(\CW,\CL,\mathcal{M})$ with depth $\CL = 7$, width $\CW = O(n+9d)$ and parameter norm $\mathcal{M}>0$, such that for any $\tilde \phi_t \in \NN(\CW,\CL,\mathcal{M})$ as defined in Eq.\ref{eq:ddpm2}, the following bound holds:
    \begin{align*}
        D_{KL}\left(q_0(x) \| \tilde p_0(x)\right)\precsim T \cdot \sup _t\left(n^{-1-\frac{3}{d}}+\frac{\log \mathcal{M}}{\mathcal{M}}\right)+ C(C_{q_0},\alpha_t,\sigma_{p,t}).
    \end{align*}
\end{theorem}

To prove the theorem, we analyze the distributional distance in Eq.\ref{eq:approx1} by decomposing it into distinct components. Each component is addressed in detail in the following section.

\hspace*{\fill}
\begin{flushleft}
  \textbf{Distance between $q_T\left(z_T\right)$ and $q\left(z_T \mid x\right)$}
\end{flushleft}

Noting that this difference is independent of the neural network's impact, it can be calculated directly from the definition of the DDPM and bounded using Bayes' rule and Jensen's inequality, which is proved in Appendix \ref{sec:Analysis of {Z}_t} in detail.
\begin{prop}
\label{prop:approx1}
    Given any $x\in [-1,1]^d$, the KL divergence between $q_T\left(z_T\right)$ and $q\left(z_T \mid x\right)$ can be bounded by a constant only depending on $\{\alpha_T, d\}$,
     \begin{align*}
          D_{{KL}}\left(q\left(z_T \mid x\right) \| q_T\left(z_T\right)\right) \leq \frac{5}{2} d \cdot \frac{\alpha_T}{1-\alpha_T}.
     \end{align*} 
\end{prop}
\begin{remarks}
   Note that the KL divergence decreases as $\alpha_T \rightarrow 0$. Since $\{\alpha_t \coloneq\prod^T_{t=1} 1-\beta_t\}$ is designed to decays exponentially with respect to the time step $T$, the conditional distribution $q\left(z_T \mid x\right)$ becomes increasingly indistinguishable from the marginal $q_T\left(z_T\right)$. In the limit as $T \rightarrow \infty, q_T\left(z_T\right)$ converges to a standard Gaussian noise distribution, and the KL divergence vanishes. This reflects the intuition that $z_T$ becomes asymptotically independent of the data $x$, representing pure noise.
\end{remarks}

\hspace*{\fill}
\begin{flushleft}
  \textbf{{Distance between $\tilde p(z_{t-1} \mid z_t)$ and $q(z_{t-1} \mid z_t, x)$}}
\end{flushleft}

The primary objective of this analysis is considered to bound the KL divergence $D_{KL}\left[q\left(z_{t-1} \mid z_t, x\right) \| \tilde p\left(z_{t-1} \mid z_t\right)\right]$. Based on the construction of the diffusion model and the domain of $z_t$, the analysis proceeds as follows.
\begin{align*}
      &E_{\mathbf{X}\sim q_0} E_{\mathbf{Z}_t \mid \mathbf{X}} D_{KL}\left[q\left(z_{t-1} \mid z_t, x\right) \|\tilde p\left(z_{t-1} \mid z_t\right)\right]\\
      =&\frac{1}{2 \sigma_{p,t}^2}E_{\mathbf{Z}_t} E_{\mathbf{X} \mid \mathbf{Z}_t}\left|\frac{1-\alpha_{t-1}}{1-\alpha_t} \sqrt{1-\beta_t} z_t+\frac{\sqrt{\alpha_{t-1}} \beta_t}{1-\alpha_t} x-\tilde{\phi}_t\left(z_t\right)\right|^2\\
    \precsim&\frac{1}{2 \sigma_{p,t}^2}(\left\{\underbrace{\int_{\left|z_t\right| \leq(\sqrt{2\xi}+\sqrt{d}+1)\sqrt{(1-{\alpha}_t)}}\left|f_\rho(z_t)-\tilde \phi_t\left(z_t\right)\right|^2d{z_t}}_{\textbf{I}}\right.\\ 
    &+\underbrace{\int_{\left|z_t\right| > (\sqrt{2\xi}+\sqrt{d}+1)\sqrt{(1-{\alpha}_t)}}\left|f_\rho(z_t)-\tilde \phi_t\left(z_t\right)\right|^2d{z_t}}_{\textbf{II}}\\
    &\left.+\underbrace{E_{\mathbf{Z}_t} E_{\mathbf{X} \mid \mathbf{Z}_t}\left|\frac{1-\alpha_{t-1}}{1-\alpha_t} \sqrt{1-\beta_t} z_t+\frac{\sqrt{\alpha_{t-1}} \beta_t}{1-\alpha_t} x-f_\rho\left(z_t\right)\right|^2}_{\textbf{III}}\right\}),\\    
\end{align*}  
where $f_\rho\in\mathbb{R}^d$ is defined as the Bayes predictor, 
\begin{equation*}
    f_\rho\left(z_t\right)=\int\left(\frac{1-\alpha_{t-1}}{1-\alpha_t} \sqrt{1-\beta_t} z_t+\frac{\sqrt{\alpha_{t-1}} \beta_t}{1-\alpha_t} x\right) \cdot q\left(x \mid z_t\right) d x.
\end{equation*}
Note that by the Bayes rule, it can be further expressed as a rational function of the form $f_\rho\coloneq \frac{f_1}{f_2}$, where
\begin{align}
\label{eq:ddpmf1}
f_1=&\int\left(\frac{1-\alpha_{t-1}}{1-\alpha_t} \sqrt{1-\beta_t} \mathbf{z}_t+\frac{\sqrt{\alpha_{t-1}} \beta_t}{1-\alpha_t} \mathbf{x}\right) \cdot q_0(\mathbf{x}) \cdot q\left(\mathbf{z}_t \mid \mathbf{x}\right)d \mathbf{x},\\
\label{eq:ddpmf2}
f_2=&q\left(\mathbf{z}_t\right)=\int q_0(\mathbf{x})q(\mathbf{z}_t \mid \mathbf{x}) d\mathbf{x}.
\end{align}
Under the squared loss, $f_\rho$ is the population-optimal predictor of the
reverse transition. Applying Bayes' rule to
$\mathbb{E}[\mathbf{X}\mid\mathbf{Z}_t=\mathbf{z}_t]$ gives the ratio
representation in Eqs.~\eqref{eq:ddpmf1}--\eqref{eq:ddpmf2}. The
approximations of $f_1$ and $f_2$ can then be combined with the SignReLU
division construction in Lemma~\ref{lemma:f1/f2sign} to approximate the
reverse-process Bayes predictor.

Since the diffusion marginal $q_t(z_t)$ may become small in the tails, we
apply the ratio approximation on the high-probability region in Part~(I)
and control its complement separately in Part~(II).

Part (I) focuses on approximating the ratio-type function $f_\rho = \frac{f_1}{f_2}$ when the input lies within the vicinity of the target distribution’s support. We conduct this analysis by extending the approximation strategy for rational functions established in Theorem~\ref{thm:f1/f2} consistent with the diffusion forward process’s noise structure. The formal result is stated below. Since the reverse transitions have prescribed Gaussian covariance, the
network-dependent part of their conditional KL divergence is proportional
to the squared approximation error
\(\|f_\rho-\widetilde{\phi}_t\|_2^2\), which connects the ratio approximation
to the distributional error in Eq.~\eqref{eq:approx1}.
 \begin{prop}
 \label{prop:z_tmin}
     At step t ($1\leq t\leq T$), there exists a SignReLU network $\NN(\CW,\CL,\mathcal{M})$ with depth $\CL=7$, width $\CW= O(n+9d)$ and parameter norm $\mathcal{M}>0$ and we define $\tilde \phi_t \in \NN(\CW,\CL,\mathcal{M})$ following Eq.\ref{eq:ddpm2}, such that for each step t in backward process,  
   \begin{align*}
    \int_{\left|z_t\right| \leq (\sqrt{2\xi}+\sqrt{d}+1)\sqrt{(1-{\alpha}_t)}} \Bigl|f_\rho-\tilde \phi_t\Bigr|^2 dz_t \precsim \frac{e^{2\xi}}{\mathcal{M}^2} \cdot n^{-1-\frac{3}{d}}.
   \end{align*}
   \end{prop}

Part (II) addresses the scenario when the latent variable of $f_\rho$ lies far beyond the support of the target distribution $q_0$. 
\begin{prop}
\label{prop:z_tmax}
Suppose $\tilde{\phi}_t$, as defined in Eq.\ref{eq:ddpm2}, belong to the SignReLU network class $\NN(\CW,\CL,\mathcal{M})$ with depth $\CL=7$, width $\CW = O(n+9d)$, and parameter norm $\mathcal{M} >0$. Then, for sufficiently large $\|z_t\|_2$, the following bound holds,
\begin{align*}
   \int_{\left|z_t\right| > (\sqrt{2\xi}+\sqrt{ d}+1)\sqrt{(1-{\alpha}_t)}}\left|f_\rho-\tilde \phi_t\right|^2d{z_t}\precsim \xi e^{-\xi}.
\end{align*}
\end{prop}

Part(III) is independent with the networks $\phi_t$ and follow the definition of the Bayes predictor $f_\rho$,
\begin{equation*}
    E_{\mathbf{Z}_t} E_{\mathbf{X} \mid \mathbf{Z}_t}(\frac{1}{2 \sigma_{p,t}^2}(\left|\frac{1-\alpha_{t-1}}{1-\alpha_t} \sqrt{1-\beta_t} z_t+\frac{\sqrt{\alpha_{t-1}} \beta_t}{1-\alpha_t} x-f_\rho\left(z_t\right)\right|^2))\leq C(C_{q_0},\sigma_{p,t}, \alpha_t).
\end{equation*}
The proofs establishing bounds for each component are provided in Appendix \ref{sec:proof of p/q}.

\hspace*{\fill}
\begin{flushleft}
  \textbf{Sketch proof of Theorem \ref{thm:approx}}
\end{flushleft}
The approximation error associated with the network $\phi$ is bounded above by a quantity derived from Propositions\ref{prop:z_tmax} and \ref{prop:z_tmin}. For all $t\in T$, the parameter $\xi \asymp \log \mathcal{M}$ is chosen such that the region
$
\left|z_t\right|\in [0,(\sqrt{2 \xi}+\sqrt{d}+1) \sqrt{1-\alpha_t}]
$
captures the primary influence of the network structure, while the complementary region
$
\left|z_t\right|\in ((\sqrt{2 \xi}+\sqrt{d}+1) \sqrt{1-\alpha_t},\infty)
$
contributes negligibly as $\mathcal{M} \rightarrow \infty$ due to exponential decay.
Choosing $\xi=\log\mathcal{M}$ balances the approximation error on the
localized region with the Gaussian-tail contribution of order
$\log\mathcal{M}/\mathcal{M}$. A complete proof of this result is provided in Appendix \ref{sec:proof of approx}.

\subsection{Estimation Error}
This section analyzes the estimation error in Eq.\ref{eq:excess}, which quantifies the error incurred when a SignReLU network approximates the target distribution from a finite sample set. We derive an upper bound for the estimation error using tools from statistical learning theory. Our main result is stated in the following theorem.
\begin{theorem}
\label{thm:esti}
 Consider the DDPM architecture under Assumption~\ref{assump:target1}, where the target data distribution is specified as $\mathbf{X} \sim q_0$. Given $\left\{x_i\right\}_{i=1}^{m}$ be $m$ i.i.d training samples drawn from $q_0$, let $\hat p_0, \tilde p_0$ be the generated distributions defined by \eqref{p0}, corresponding to the SignReLU neural network families $\{\tilde \phi_t\}\in \CN\CN(\CW,\CL,\mathcal{M})$ (defined by \eqref{eq:ddpm1}) and $\{\hat \phi_t\}\in \CN\CN(\CW,\CL,\mathcal{M})$( defined by \eqref{eq:ddpm2}), respectively. Here $\CN\CN(\CW,\CL,\mathcal{M})$ denotes a SignReLU network class with depth $\CL = 7$, width $\CW = O(n+9d)$ and parameter norm $\mathcal{M}>0$. Then with probability at least $1 - 2\delta$,
\begin{align*}
          &D_{KL}\left(q_0(x) \| \hat{p}_0(\hat z_0)\right)-D_{KL}\left(q_0(x) \| \tilde{p}_0(x)\right)\\
		&\precsim T\left(\mathcal{M}^2+\log \mathcal{M}+T^3\mathcal{M}^2\log \mathcal{M}\right)(\sqrt{\frac{7n^2}{m}}+\sqrt{\frac{2\log \frac{1}{\delta}}{ m}}).
\end{align*}
\end{theorem}
To establish the theorem, the estimation error is analyzed by the following decomposition
\begin{align}
\label{eq:esti}
    &D_{KL}\left(q_0(x) \| \hat{p}_0(\hat z_0)\right)-D_{KL}\left(q_0(x) \| \tilde{p}_0(x)\right)\notag\\
        =& \underbrace{E_{x\sim q_0} \log \tilde{p}_0(x) - \frac{1}{m}\sum_{i=1}^m \log \tilde{p}_0({x_i})}_{\mathrm{I}}+\underbrace{\frac{1}{m}\sum_{i=1}^m \log \hat{p}_0(\hat z_{0,i})-E_{\hat z_0\sim \hat p_0} \log \hat p_0(\hat z_0)}_{\mathrm{II}}\notag\\
        &+\underbrace{\frac{1}{m}\sum_{i=1}^m \log \tilde{p}_0(x_i)-\frac{1}{m}\sum_{i=1}^m \log \hat{p}_0(\hat z_{0,i})}_{\mathrm{III}}.
\end{align}
This decomposition of the estimation error indicates that the statistical error, denoted as $E[d_\phi(x,x_m)]\coloneq E_{\mathbf{X}} \log {p}_0(x) - \frac{1}{m}\sum_{i=1}^m \log{p}_0(x_i)$, is its primary component. A key requirement for this analysis is the boundedness of $\|\log\tilde p_0(x)\|_\infty$, which is necessary for the ideal reverse process, and $\|\log\hat p_0(\hat z_0)\|_\infty$, which arises in the practical implementation of the backward process. 
\begin{prop}
\label{prop:logbound}
Suppose we implement $\phi_t\in\CN\CN(\CW,\CL,\mathcal{M})$with depth $\CL = 7$, width $\CW = O(n+9d)$ and parameter norm $\mathcal{M}>0$, we have
   \begin{align*}
       &B_{\tilde p}\coloneq\|\log \tilde p_0(x)\|_\infty \leq O\left(T\left(\mathcal{M}^2+\log \mathcal{M}\right)\right),\\
       &B_{\hat p}\coloneq\|\log\hat p_0(\hat z_0)\|_\infty\leq O\left(T^4\mathcal{M}^2\log \mathcal{M}\right).
   \end{align*}
\end{prop}
An upper bound on the statistical error is derived in terms of the covering numbers associated with the relevant function classes.
\begin{mydef}[Covering number]
\label{def:cov}
    Let $(S, \rho)$ be a metric space, and let $T \subset S$. We say that $T^{\prime} \subset S$ is an $\alpha$-cover for $T$ if, for all $x \in T$, there exists $y \in T^{\prime}$ such that $\rho(x, y) \leq \alpha$. The $\alpha$-covering number of $(T, \rho)$, denoted $\mathcal{N}(\alpha, T, \rho)$ is the size of the smallest $\alpha$-covering.
\end{mydef}
Based on this definition, an upper bound on the statistical error of DDPM is obtained using the covering numbers of the function classes corresponding to $\log \tilde{p}_0(x)$ and $\log \hat{p}_0(\hat{z}_0)$.
\begin{prop}
    \label{prop:staterror}
Let $h(x) \coloneqq \log \tilde{p}_0(x)$ and $g(\hat{z}_0) \coloneqq \log \hat{p}_0(\hat{z}_0)$. Assume that the function classes $\mathcal{H}$ and $\mathcal{G}$ satisfy the uniform boundedness conditions $\sup_{h \in \mathcal{H}}\|h\|_{\infty} \precsim B_{\tilde p}$, $\sup_{g \in \mathcal{G}}\|g\|_{\infty} \precsim B_{\hat p}$. Then, following Eq.\ref{eq:esti}, with probability at least $1 - 2\delta$, where $\delta = \max{\delta_1, \delta_2}$, the estimation error satisfies
\begin{align*}
   &D_{KL}\left(q_0(x) \| \hat{p}_0(\hat z_0)\right)-D_{KL}\left(q_0(x) \| \tilde{p}_0(x)\right)\\
   \leq& E_{\mathbf{X}} \inf _{0<\xi_1< B_{\tilde p}}\left(4 \xi_1+\frac{12}{\sqrt{m}} \int_{\xi_1}^{ B_{\tilde p}} \sqrt{\log \mathcal{N}\left(\varepsilon, \mathcal{\hat H}, \|\cdot\|_{\infty}\right)} d \epsilon\right)+B_{\tilde p}\sqrt{\frac{\log \frac{1}{\delta_1}}{ m}}\\
    &+E_{\mathbf{\hat Z_0}} \inf _{0<\xi_2< B_{\hat p}}\left(4 \xi_2+\frac{12}{\sqrt{m}} \int_{\xi_2}^{ B_{\hat p}} \sqrt{\log \mathcal{N}\left(\varepsilon, \mathcal{\hat G}, \|\cdot\|_{\infty}\right)} d \epsilon\right)+B_{\hat p}\sqrt{\frac{\log \frac{1}{\delta_2}}{ m}},
\end{align*}
 where $\mathcal{\hat H}=\left\{\left(h\left(x_{1}\right), \ldots, h\left(x_{m}\right)\right): h \in \mathcal{H}\right\}$ and $\mathcal{\hat G}=\left\{\left(g\left(\hat z_{0,1}\right), \ldots, g\left(\hat z_{0,m}\right)\right): g \in \mathcal{G}\right\}$ for any i.i.d. samples $\left\{x_{i}\right\}_{i=1}^m$ from $q_0$ and $\left\{\hat z_{0,i}\right\}_{i=1}^m$ from $\hat p_0$. $\mathcal{N}\left(\varepsilon, \mathcal{\hat H}, \|\cdot\|_{\infty}\right)$ and $\mathcal{N}\left(\varepsilon, \mathcal{\hat G}, \|\cdot\|_{\infty}\right)$ are the $\epsilon$-covering number of $\mathcal{\hat H} \subseteq \mathbb{R}^d$ and $\mathcal{\hat G} \subseteq \mathbb{R}^d$ with respect to the $\|\cdot\|_{\infty}$ distance.  
\end{prop}
The detailed proofs of Proposition \ref{prop:logbound} and Proposition \ref{prop:staterror} postponed in Appendix~\ref{sec:proof of Statistical error}.
\paragraph{Sketch proof of Theorem \ref{thm:esti}} The analysis begins by examining the covering numbers of the function classes $\mathcal{\hat H}$, $\mathcal{\hat G}$, which correspond to $\log \tilde{p}_0(x)$ and $\log \hat{p}_0(\hat{z}_0)$, respectively. These covering numbers are bounded as,
\begin{align*}
	\mathcal{N}\left(\epsilon,\mathcal{\hat H}, \|\cdot\|_{\infty}\right) \leq C\left(B_{\hat p} / \epsilon\right)^{n^2}, \quad
	\mathcal{N}\left(\epsilon,\mathcal{\hat G}, \|\cdot\|_{\infty}\right) \leq C\left(B_{\tilde p} / \epsilon\right)^{n^2}.
\end{align*}
Substituting these bounds into Proposition~\ref{prop:staterror} yields the main result regarding the estimation error.
\paragraph{Discussion on the training set size} To reduce estimation error, we augment each training sample with multiple latent trajectories. Specifically, during the forward process, we generate $m_z$ independent latent trajectories $\left\{z_{t, i, j}\right\}_{j=1}^{m_z}$ at each time step $t$ for every data point. In the backward process, $m_z$ independent latent variables $\left\{\hat{z}_{T, i, j}\right\}_{j=1}^{m_z}$, where $\hat{z}_{T, i, j} \stackrel{\text { i.i.d. }}{\sim} \mathcal{N}\left(0, \mathbf{I}_d\right)$, are sampled at the initial step for each training example $x_i$. As a result, when the neural network is constructed in accordance with Theorem \ref{thm:esti}, the estimation error is reduced. This improvement arises because the noise-induced error at time steps $t \geq 2$ is suppressed to a level comparable to that achieved with an effective sample size of $m$. A detailed proof is provided in Appendix \ref{sec:proof of data size}.
\begin{corollary}
\label{coro:data}
Let \(\{x_i\}_{i=1}^{m}\) be \(m\) i.i.d.\ training samples drawn from \(q_0\). In the forward process, augment each sample with \(m_z\) latent variables \(\{z_{t,j}\}_{j=1}^{m_z}\). For each \(i \in \{1, \ldots, m\}\), draw \(\{\hat{z}_{T,i,j}\}_{j=1}^{m_z}\) independently from \(\mathcal{N}(0, \mathbf{I}_d)\), and use them as inputs to a neural network \(\hat{\phi}_t \in \mathcal{NN}(\mathcal{W}, \mathcal{L}, \mathcal{M})\) with depth \(\mathcal{L} = 7\), width \(\mathcal{W} = O(n + 9d)\), and parameter norm bounded by \(\mathcal{M} > 0\). Then, with probability at least \(1 - 2\delta\), the estimation error satisfies
\begin{align*}
          &D_{KL}\left(q_0(x) \| \hat{p}_0(\hat z_0)\right)-D_{KL}\left(q_0(x) \| \tilde{p}_0(x)\right)\\
		&\precsim T\left(\mathcal{M}^2+\log \mathcal{M}\right)(\sqrt{\frac{7n^2}{ m}}+\sqrt{\frac{2\log \frac{1}{\delta}}{m}})+\left(T^4\mathcal{M}^2\log \mathcal{M}\right)(\sqrt{\frac{7n^2}{m_z^2 m}}+\sqrt{\frac{2\log \frac{1}{\delta}}{ m_z^2 m}}).
\end{align*}
Moreover, if \(m_z = O(T^3)\), the error reduces to
\begin{align*}
          &D_{KL}\left(q_0(x) \| \hat{p}_0(\hat z_0)\right)-D_{KL}\left(q_0(x) \| \tilde{p}_0(x)\right)\\
		&\precsim T\left(\mathcal{M}^2+\log \mathcal{M}+\mathcal{M}^2\log \mathcal{M}\right)(\sqrt{\frac{7n^2}{ m}}+\sqrt{\frac{2\log \frac{1}{\delta}}{ m}}).
\end{align*}
\end{corollary}

  \subsection{Excess KL Risk}
  Having established individual bounds for the approximation error (Theorem~\ref{thm:approx}) and the estimation error (Theorem~\ref{thm:esti}), we proceed to integrate these results to characterize the overall performance of the DDPM. Specifically, by leveraging the decomposition of the excess Kullback–Leibler (KL) risk as formulated in Eq.\ref{eq:excess}, we derive a unified upper bound that quantifies the total deviation of the learned model from the target distribution. This bound captures the combined effects of approximation and statistical errors, thereby providing a comprehensive theoretical guarantee for the excess KL risk incurred by the DDPM.
  \begin{theorem} 
  \label{thm:excess}
  Consider the DDPM architecture under Assumption~\ref{assump:target1}, where the target data distribution satisfies $\mathbf{X} \sim q_0$. Let $\left\{x_i\right\}_{i=1}^{m}$ be $m$ i.i.d training samples taken from $q_0$. 
  Let $\hat p_0$ be the generated distribution corresponding to the SignReLU neural network families $\{\hat \phi_t\}\in \CN\CN(\CW,\CL,\mathcal{M})$ with depth $\CL = 7$, width $\CW = O(n+9d)$ and parameter norm $\mathcal{M}=O(n^{1+\frac{3}{d}}\log n)$. If $n, T$ are selected to satisfy $ m= O(T^6n^{8+\frac{18}{d}}(\log n)^6)$, then, with probability at least $1 - 2\delta$, the following excess KL risk bound holds:
     \begin{equation*}
     D\left(q_0(x) \| \hat{p}_0(\hat{z}_0)\right)\precsim T(n^{-1-\frac{3}{d}}+\sqrt{\log\frac{1}{\delta}}n^{-2-\frac{3}{d}}).
     \end{equation*}
\end{theorem}
\begin{proof}
Restating the results from Theorem \ref{thm:approx} and Theorem \ref{thm:esti},
    \begin{align}
    \label{eq:excess1}
         D_{KL}\left(q_0(x) \| \tilde p_0(x)\right)&\precsim T \cdot \sup _t\left(n^{-1-\frac{3}{d}}+\frac{\log \mathcal{M}}{\mathcal{M}}\right)+ C(C_{q_0},\alpha_t,\sigma_{p,t})\\
    \label{eq:excess2}
        D_{KL}\left(q_0(x) \| \hat{p}_0(\hat z_0)\right)&-D_{KL}\left(q_0(x) \| \tilde{p}_0(x)\right),\notag\\
		&\precsim T\left(\mathcal{M}^2+\log \mathcal{M}+T^3\mathcal{M}^2\log \mathcal{M}\right)(\sqrt{\frac{7n^2}{m}}+\sqrt{\frac{2\log \frac{1}{\delta}}{ m}}).
    \end{align}
   To balance the approximation error, as given in Eq.~\ref{eq:excess1}, and the estimation error, as characterized in Eq.~\ref{eq:excess2}, we analyze the required sample size $m$ when 
    \begin{align*}
       T\left(\mathcal{M}^2+\log \mathcal{M}+T^3\mathcal{M}^2\log \mathcal{M}\right)(\sqrt{\frac{7n^2}{m}}+\sqrt{\frac{2\log \frac{1}{\delta}}{ m}})\asymp T \cdot \sup _t\left(n^{-1-\frac{3}{d}}+\frac{\log \mathcal{M}}{\mathcal{M}}\right).
    \end{align*}
Take $\mathcal{M}\asymp n^{1+\frac{3}{d}}\log n$,
\begin{align*}
       T^3n^{2+\frac{6}{d}}(\log n)^3\left(\sqrt{\frac{7n^2}{m}}\right)\asymp  \sup _t\left(n^{-1-\frac{3}{d}}\right).
    \end{align*}
   Therefore, the relation between the estimation error and approximation error is satisfied as $m \asymp T^6n^{8+\frac{18}{d}}(\log n)^6$,
   \begin{align*}
       D\left(q_0(x) \| \hat{p}_0(\hat{z}_0)\right)\precsim T(n^{-1-\frac{3}{d}}+\sqrt{\log\frac{1}{\delta}}n^{-2-\frac{3}{d}}).
   \end{align*}
\end{proof}

As shown in Corollary \ref{coro:data}, generating multiple independent latent trajectories during the forward process reduces the estimation error. This allows the same excess Kullback–Leibler (KL) risk bound to be achieved with fewer training samples, provided the network architecture remains unchanged.

\begin{corollary}
    Let $\left\{x_i\right\}_{i=1}^m$ be $m$ i.i.d. training samples drawn from $q_0$. In the forward process, augment each sample with $m_z=O\left(T^3\right)$ latent variables $\left\{z_{t, j}\right\}_{j=1}^{m_z}$. For each $i \in\{1, \ldots, m\}$, draw $\left\{\hat{z}_{T, i, j}\right\}_{j=1}^{m_z}$ independently from $\mathcal{N}\left(0, \mathbf{I}_d\right)$, and use them as inputs to the neural network $\hat{\phi}_t \in \mathcal{N} \mathcal{N}(\mathcal{W}, \mathcal{L}, \mathcal{M})$, where $\mathcal{L}=7$, $\mathcal{W}=O(n+9 d)$, and $\mathcal{M}=O\left(n^{1+\frac{3}{d}} \log n\right)$. If $n$ and $T$ are chosen such that the training sample size satisfies $m=O\left(n^{8+\frac{18}{d}}(\log n)^6\right)$, then the excess KL risk bound holds as stated in Theorem \ref{thm:excess}.   
\end{corollary}

\section{Numerical Experiments}
\label{sec:experiments}

We complement the theoretical analysis with a numerical study on CIFAR-10.  The study
compares SignReLU with the commonly used SiLU and ReLU activations while keeping the
model architecture, optimization procedure, checkpoint, and sampling protocol fixed.

\subsection{Experimental setup}
\label{subsec:experimental-setup}

We train denoising diffusion probabilistic models on CIFAR-10 using the standard
noise-prediction parameterization.  Given a clean image $x_0$, a time index $t$, and
$\epsilon\sim\mathcal{N}(0,I)$, the network predicts the injected noise from
$x_t=\sqrt{\alpha_t}x_0+\sqrt{1-\alpha_t}\epsilon$ by minimizing
\begin{equation}
    \mathcal{L}_{\epsilon}(\theta)
    =\mathbb{E}_{x_0,t,\epsilon}
      \left[\left\|\epsilon-\epsilon_\theta(x_t,t)\right\|_2^2\right].
    \label{eq:epsilon-training-objective}
\end{equation}
This is the conventional DDPM training objective and is used identically for all three
activation functions.

We use the activation definitions given in Section~2, with the negative-branch
coefficient set to $\alpha=0.2$ for SignReLU.  All models share the same U-Net
architecture and are trained for $2{,}040$ epochs using $1{,}000$ diffusion steps, a
linear variance schedule, and the mean-squared noise-prediction loss
in~\eqref{eq:epsilon-training-objective}.  The optimizer, batch size, dropout, and EMA
settings are also held fixed across activations.

For evaluation, all three checkpoints use the same $100$-step DDIM sampler with a
quadratic timestep subsequence and EMA parameters.  Fr\'{e}chet inception distance
(FID), precision, and recall are computed from $50{,}000$ generated samples.  DDIM
reuses the denoiser learned under the DDPM noise-prediction objective
in~\eqref{eq:epsilon-training-objective}.

\subsection{Results}
\label{subsec:experimental-results}

Table~\ref{tab:cifar-activation-comparison} presents the controlled
comparison of the three activation functions.

\begin{table}[!htbp]
    \centering
    \caption{Activation-function comparison on CIFAR-10 using the common
    $100$-step sampling protocol and $50{,}000$ generated samples.}
    \label{tab:cifar-activation-comparison}

    \small
    \setlength{\tabcolsep}{16pt}
    \renewcommand{\arraystretch}{1.1}
    \begin{tabular}{lccc}
        \toprule
        Activation
        & FID $\downarrow$
        & Precision $\uparrow$
        & Recall $\uparrow$ \\
        \midrule
        SignReLU & \textbf{3.87} & 0.718          & \textbf{0.547} \\
        SiLU     & 4.10          & \textbf{0.734} & 0.532          \\
        ReLU     & 4.11          & 0.731          & 0.523          \\
        \bottomrule
    \end{tabular}
\end{table}

SignReLU obtains the lowest FID, $3.87$, compared with $4.10$ for SiLU
and $4.11$ for ReLU. Relative to SiLU and ReLU, this corresponds to FID
reductions of approximately $5.6\%$ and $5.8\%$, respectively. SiLU
achieves the highest precision, whereas SignReLU achieves the highest
recall. In particular, SignReLU improves recall from $0.532$ to $0.547$
relative to SiLU and from $0.523$ to $0.547$ relative to ReLU. Overall,
the results show a modest improvement in FID and recall under the common
sampling protocol, accompanied by a small reduction in precision.

\begin{figure}[!htbp]
    \centering
    \includegraphics[
        width=0.78\linewidth,
        trim=5 5 5 5,
        clip
    ]{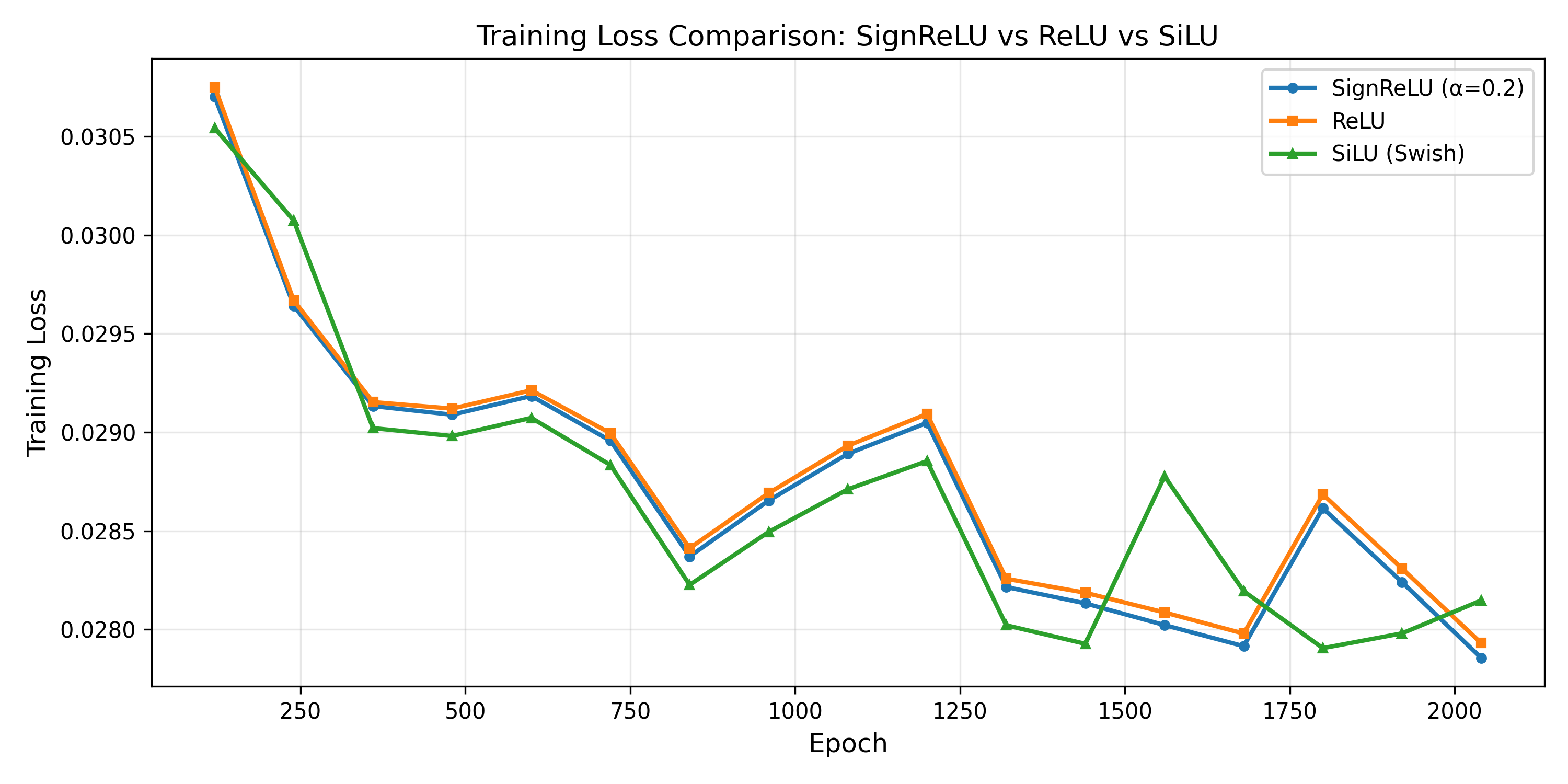}
    \caption{CIFAR-10 training-loss trajectories for SignReLU, SiLU,
    and ReLU under the common training configuration.}
    \label{fig:cifar-training-loss}
\end{figure}

\FloatBarrier

Overall, the results indicate that SignReLU can be incorporated into a
conventional noise-prediction diffusion pipeline and performs
competitively with standard activations under the common experimental
protocol. These observations provide an empirical complement to the
preceding analysis.

\section{Conclusion}
This paper establishes approximation guarantees for a general class of ratio functionals $\frac{f_1}{f_2}$, where ${f_1}$ and ${f_2}$ are kernel-induced marginal densities. We approximate this ratio using SignReLU neural networks, explicitly controlling the approximation error in both the numerator and denominator to ensure stability, particularly near vanishing denominators. Our theoretical framework accommodates denominators that are kernel-induced marginal integrals, thereby capturing the ratio structures inherent in diffusion training objectives—such as score matching and denoising—and enabling stable control through architectural constraints and regularization.

We instantiate this framework for diffusion models. For the standard diffusion model, Denoising Diffusion Probabilistic Model (DDPM), we derive a direct bound on the Kullback–Leibler divergence $D_{KL}(q_0(x) \| \hat{p}_0(\hat z_0))$ between the data distribution and the learned model. We demonstrate that, when the ground-truth density belongs to a Hölder class,
suitably constructed SignReLU networks achieve explicit convergence rates.
Furthermore, Theorem~\ref{thm:excess} establishes an excess-risk
upper bound for DDPM and specifies a SignReLU architecture that attains this
bound for a given sample size and diffusion-time discretization. This provides an end-to-end pathway from the $\frac{f_1}{f_2}$ approximation error to the final diffusion excess risk under realistic training protocols.

Although developed for standard diffusion models, our ratio-based analysis is modular and extends naturally to accelerated samplers and learned noise schedules. We anticipate that the same $\frac{f_1}{f_2}$ approximation–estimation blueprint will yield improved excess risk bounds for these variants, a direction we leave for future work.

\section{Acknowledgement} 
The authors thank the anonymous reviewers for their constructive comments and suggestions. This work was partially supported by the Research Grants Council of Hong Kong under Project Nos. CityU 11315522 and 11300525.
During the preparation of this work, the authors used GPT-5.5 solely for English-language editing and grammar checking. The authors reviewed and revised all generated suggestions as needed and take full responsibility for the content of the published article.
\newpage
\bibliographystyle{siamplain}
  \bibliography{main}

\newpage
\appendix
\section{Analysis on "$\frac{f_1}{f_2}$ - type" functions}
\subsection{Proof of Proposition \ref{prop:kernel}}
\label{sec:proof of kernel}
As defined in Eq.\ref{eq:pre1}, we consider a class of single-hidden-layer neural networks with parameters constrained to ensure both expressiveness and stability. Specifically, we define the function class:
\begin{align*}
\mathcal{N} \mathcal{N}(\mathcal{W}, 1, \mathcal{M})=\left\{\sum_{j=1}^n a_j \sigma\left(\omega_j \cdot x+b_j\right):\left|\omega_j\right|=1, b_j \in\left[c_1, c_2\right], \sum_{j=1}^n\left|a_j\right| \leq \mathcal{M}\right\}.
\end{align*}
The constants $c_1$ and $c_2$ are chosen such that,
\begin{align*}
c_1<\inf \{\omega \cdot x: x \in \Omega\}<\sup \{\omega \cdot x: x \in \Omega\}<c_2,
\end{align*}
ensuring that the affine transformation $\omega_j \cdot x + b_j$ maps the input into the active region of the activation function for all $x \in \Omega$. The approximation capabilities of this network class are established in Lemma \ref{lemma:siegel}, which provides bounds on the rate at which networks in $\mathcal{NN}(\mathcal{W}, 1, \mathcal{M})$ can approximate target functions from the class $\mathcal{B}$. Here, $\mathcal{B}$ denotes the closure of the convex, symmetric hull of ridge functions of the form $\sigma(\omega \cdot x + b)$, explicitly given by:
\begin{align*}
    \mathcal{B}=\overline{\left\{\sum_{j=1}^n a_j \sigma\left(\omega_j \cdot x+b_j\right): |\omega_j|=1, b_j \in\left[c_1, c_2\right], \sum_{i=1}^n\left|a_i\right| \leq 1\right\}}.
\end{align*}
\begin{lemma}[Approximation Rates for Smoothly Parameterized Dictionaries\cite{siegel2024sharp}]
\label{lemma:siegel}
Let $2 \leq p<\infty$. Then there exists $\CN\CN(\CW,1,\mathcal{M})$ with $\CW=n$ and $\mathcal{M}>0$ such that for all $f \in \mathcal{B}$ we have
\begin{align*}
    \inf _{f_n \in \CN\CN(\CW,1,\mathcal{M})}\left|f-f_n\right| \precsim n^{-\frac{1}{2}-\frac{p+1}{p d}}.
\end{align*}
\end{lemma}
Having introduced the function class $\mathcal{S}$ in Eq.\ref{eq:set}, where
\begin{align*}
	\mathcal{S}=\left\{f(x)=\int_{\Omega} \Phi(x, y) g(y) d y: g \in L^1(\Omega), \Phi(x, y)=\sum_{j=1}^m \phi_j(x \cdot y),\left\|\phi_j^{\prime \prime}\right\|_{\infty}<\infty\right\},
\end{align*}
our next objective is to demonstrate that any $f \in \mathcal{S}$ also belongs to $\mathcal{B}$. This inclusion is crucial, as it allows us to apply the approximation rate of Lemma \ref{lemma:siegel} to functions in $\mathcal{S}$ via the network class $\mathcal{N} \mathcal{N}(\mathcal{W}, 1, \mathcal{M})$.\begin{lemma}
\label{lemma:convex-hull}
Let $\Omega, \tilde{\Omega} \subset \mathbb{R}^d$ be compact. If $\psi$ is a universal function with $\psi^{\prime \prime}$ existing and

\begin{equation}
    f(x)=\int_{\Omega} \psi(u(x,y)) g(y) d y, \quad \forall x \in \tilde{\Omega},
    \label{eq:form1}
\end{equation}

with $|u(x,y)| \leq c$ and $g \in L^1(\Omega)$, then

\begin{align}
    f(x)= & \int_{\Omega} \int_0^c \left[(u(x,y)-t)_{+} g(y) \psi^{\prime \prime}(t) +(-u(x,y)-t)_{+} g(y) \psi^{\prime \prime}(-t)\right] d t d y\notag\\
    &+ \int_{\Omega}u(x,y)\psi^{\prime}(0)g(y)+ \psi(0)g(y)dy
    \label{eq:form2}.
\end{align}

\end{lemma}

\begin{proof}
We claim that for any $c>0$ and $|u| \leq c$,

$$
\psi(u)=\int_0^c\left[(u-t)_{+} \psi^{\prime \prime}(t)+(-u-t)_{+} \psi^{\prime \prime}(-t)\right] d t+u \psi^{\prime}(0)+\psi(0).
$$

To verify this, we consider two cases when $u \geq 0$ and $u<0$. 

For $0 \leq u<c$,
$$
\begin{aligned}
& \int_0^c\left[(u-t)_{+} \psi^{\prime \prime}(t)+(-u-t)_{+} \psi^{\prime \prime}(-t)\right] d t=\int_0^c(u-t)_{+} \psi^{\prime \prime}(t) d t \\
= & \int_0^u(u-t) \psi^{\prime \prime}(t) d t=\left.(u-t) \psi^{\prime}(t)\right|_0 ^u+\int_0^u \psi^{\prime}(t) d t \\
= & \psi(u)-\psi(0)- u\psi^{\prime}(0) .
\end{aligned}
$$

For $-c<u<0$,
$$
\begin{aligned}
& \int_0^c\left[(u-t)_{+} \psi^{\prime \prime}(t)+(-u-t)_{+} \psi^{\prime \prime}(-t)\right] d t=\int_0^c(-u-t)_{+} \psi^{\prime \prime}(-t) d t \\
= & \int_0^{-u}-(u+t) \psi^{\prime \prime}(-t) d t=\left.(u+t) \psi^{\prime}(-t)\right|_0 ^{-u}-\int_0^{-u} \psi^{\prime}(-t) d t \\
= & \psi(u)-\psi(0)-u \psi^{\prime}(0) .
\end{aligned}
$$

Then substituting $\psi$ in Eq.\ref{eq:form1} will lead to Eq.\ref{eq:form2}.
\end{proof}
For a given $f \in \mathcal{S}$, define the function $u(x, y)=x \cdot y$, and note that $|u(x, y)| \leq 1$ for all $x, y \in \Omega$. We can then express $f(x)$ as follows:
\begin{align*}
	f(x)&=\int_{\Omega} \Phi(x, y) g(y) d y\\
	&=\int_{\Omega}\sum_{j=1}^m (\phi_j(x \cdot y)g(y)) d y\\
	&=\int_{\Omega}\sum_{j=1}^m (\int_0^1 \left[(u(x,y)-t)_{+} g(y) \phi_j^{\prime \prime}(t) +(-u(x,y)-t)_{+} g(y) \phi_j^{\prime \prime}(-t)\right] d t) d y.
\end{align*}
This specific integral representation allows us to conclude that $f \in \mathcal{B}$. Consequently, by applying Lemma \ref{lemma:siegel} to $f$, the proof of Proposition \ref{prop:kernel} is complete.

\subsection{Proof of Theorem \ref{thm:f1/f2}}
\label{sec:proof of thm1}
For any $f_1,f_2\in S$, Proposition~\ref{prop:kernel} guarantees the existence of shallow SignReLU networks $\phi_{f_1} \in \mathcal{N}\mathcal{N}(\mathcal{W}_1, 1, \mathcal{M}_1)$ and $\phi_{f_2} \in \mathcal{N}\mathcal{N}(\mathcal{W}_2, 1, \mathcal{M}_2)$, with widths $\mathcal{W}_1 = n_1$ and $\mathcal{W}_2 = n_2$, and parameter norms $\mathcal{M}_1>0$ and $\mathcal{M}_2>0$, such that
\begin{align*}
   &\inf _{\phi_{f_1} \in \CN\CN(\CW_1,1,\mathcal{M}_1)}\left|f_1-\phi_{f_1}\right| \precsim n_1^{-\frac{1}{2}-\frac{3}{2d}},\\
   &\inf _{\phi_{f_2} \in \CN\CN(\CW_2,1,\mathcal{M}_2)}\left|f_2-\phi_{f_2}\right| \precsim n_2^{-\frac{1}{2}-\frac{3}{2d}}.
\end{align*}

By Lemma~\ref{lemma:f1/f2sign}, there is a SignReLU network ${\phi}$ with $\alpha=1$ in Eq.\ref{eq:signrelu}, with depth $\mathcal{L}=6$, width $\mathcal{W}\leq 9$, number of weights $\mathcal{N}\leq 71$ computing $(x,y)\mapsto y/x$ on $[c,C]\times[-C,C]$. We define the final network component-wise by
$$\phi(x):=\bigl(\phi(\phi_{f_2}(x),\phi_{f_1}(x))\bigr)^\top.$$
For simplicity, we assume that $|f_1(x)| \leq C$ and $c \leq |f_2(x)| \leq C$ for all $x \in \mathbf{X}$, where $0 < c < C$ are constants. Then, point-wise in $\mathbf{X}$, the following holds:
\begin{align*}
    \left|\frac{f_1}{f_2}-\frac{\widetilde{f}_1}{\widetilde{f}_2}\right| \leq \frac{1}{\left|\widetilde{f}_2\right|}\left|f_1-\widetilde{f}_1\right|+\left|f_1\right|\left|\frac{1}{f_2}-\frac{1}{\widetilde{f}_2}\right|.
\end{align*}
Using the bounds above yields
\begin{align*}
    \left|\frac{f_1}{f_2}-\phi\right| \leq \frac{1}{\mathcal{M}_2}\left|f_1-\widetilde{f}_1\right|+\frac{C}{c\mathcal{M}_2}\left|f_2-\widetilde{f}_2\right|  \precsim \frac{1}{\mathcal{M}_2} n_1^{-\frac{1}{2}-\frac{3}{2d}}+\frac{C}{c\mathcal{M}_2} n_2^{-\frac{1}{2}-\frac{3}{2d}} . 
\end{align*}
Taking $n=\min\{n_1,n_2\}$, we obtain.
\begin{align*}
    \left|\frac{f_1}{f_2}-\phi\right| \precsim \frac{C}{c\mathcal{M}_2} n^{-\frac{1}{2}-\frac{3}{2d}}.
\end{align*}

We now analyze the complexity of the constructed network. The networks $\phi_{f_1}$ and $\phi_{f_2}$ are shallow (of constant depth $\mathcal{L} = 1$), with widths $n_1$ and $n_2$, respectively. The division sub-network contributes an additional depth of $6$, a width of at most $9$, and no more than $71$ parameters. Consequently, the final network $\phi$ has depth $\mathcal{L} \leq 7$ and width $\mathcal{W} = O(n + 9)$, where $n = \max\{n_1, n_2\}$. Since the parameters of $\phi_{f_1}$ and $\phi_{f_2}$ are also norm-bounded, the total parameter norm depends only on $p$ and $d$. Given that $\mathcal{M}_1$ and $\mathcal{M}_2$ are positive constants, the total parameter norm $\mathcal{M}$ remains a bounded positive constant, as it results from adding a finite number of bounded parameters.

Thus, the final network $\phi$ can be realized with depth $\mathcal{L} \leq 7$, width $\mathcal{W} = O(n + 9)$, and parameter norm $\mathcal{M}>0$. This concludes the proof.

\section{Approximation error}
\subsection{Decomposition of approximation error}
\label{sec:Decomposition of approx}
By the definition of the DDPM, we can write $D_{KL}\left(q_0(x) \| \tilde{p}_0(x)\right)$ as,
\begin{align*}
        &D_{KL}\left(q_0(x) \| \tilde p_0(x)\right)\\
        =&  \int q_0(x) \log \frac{q_0(x)}{\int q_T\left(z_T\right) \cdot \tilde p\left(x \mid z_1\right) \cdot \prod\limits_{t=2}^{T} \tilde p\left(z_{t-1} \mid z_t\right) d z_{1}, \ldots, z_T} d x \\
        =&H\left[q_0(x)\right]-E_{x\sim q_0}\log \int q_T\left(z_T\right) \cdot \tilde p\left(x \mid z_1\right) \cdot \prod_{t=2}^T \tilde p\left(z_{t-1} \mid z_t\right) d z_1, \ldots, z_T\\
  = & H\left[q_0(x)\right]-E_{x\sim q_0} \log \int q\left(z_1, \ldots, z_T \mid x\right)  \frac{ q_T\left(z_T\right) \cdot\tilde p\left(x \mid z_1\right) \cdot \prod\limits_{t=2}^T \tilde p\left(z_{t-1}\mid z_{t}\right)}{q\left(z_1, \ldots, z_t \mid x\right)} d z_1, \ldots, z_T,
        \end{align*}
        where $H_0(q_0)=\int q_0(x)\log q_0(x)dx$ is the entropy of the target distribution.       By inserting the conditional density $q(z_1,\ldots, z_T\mid x)$ and applying Jensen's inequality, one can have that
      \begin{align*}      
      &D_{KL}\left(q_0(x) \| \tilde p_0(x)\right)\\
\leq&H\left[q_0(x)\right]-E_{x\sim q_0}  \int q\left(z_1, \ldots, z_T \mid x\right) \log \frac{ q_T\left(z_T\right) \cdot \tilde p\left(x \mid z_1\right) \cdot \prod\limits_{t=2}^T \tilde p\left(z_{t-1}\mid z_{t}\right)}{q\left(z_1, \ldots, z_t \mid x\right)} d z_1, \ldots, z_T\\
= &H\left[q_0(x)\right]-E_{x \sim q_0} E_{z_1 \mid x} \log \tilde p\left(x \mid z_1\right)\\
&+E_{x\sim{q_0}} \int q\left(z_T \mid x\right) \cdot \prod\limits_{t=2}^Tq\left(z_{t-1} \mid z_t, x\right)q(z_t\mid x) \log \frac{q\left(z_T \mid x\right) \cdot \prod\limits_{t=2}^Tq\left(z_{t-1} \mid z_t, x\right) }{q_T\left(z_T\right) \cdot \prod\limits_{t=2}^T \tilde p\left(z_{t-1} \mid z_t\right)} d z_1, \ldots, z_T.
\end{align*}
where the last step you use the Bayes' rule $q(z_1,\ldots, z_T\mid x)=q\left(z_T \mid x\right) \cdot \prod\limits_{t=2}^Tq\left(z_{t-1} \mid z_t, x\right)$.
Then by splitting the integral, we have
\begin{align*}
&D_{KL}\left(q_0(x) \| \tilde p_0(x)\right)\\
\leq  &H\left[q_0(x)\right]+E_{x\sim q_0} E_{z_1 \mid x}  \frac{\left|x-\tilde \phi_1\left(z_1\right)\right|^2}{2 \sigma_{p,1}^2}+\frac{d}{2} \log \left(2 \pi \sigma_{p,1}^2\right) \\
&+E_{x\sim q_0} D_{KL}\left[q\left(z_T \mid x\right) \| q_T\left(z_T\right)\right]+ \sum_{t=2}^T E_{x\sim q_0} E_{z_t \mid x} D_{KL}\left[q\left(z_{t-1} \mid z_t, x\right) \|  \tilde p\left(z_{t-1} \mid z_t\right)\right].
    \end{align*}
\subsection{Distance between $q(z_{T})$ and $q(z_{T} \mid x)$}
\label{sec:Analysis of {Z}_t}
\paragraph{Proof of Proposition \ref{prop:approx1}}
    By definition,
\begin{align*}
D_{{KL}}\!\bigl(q(z_{T}\mid x)\,\|\,q_T(z_T)\bigr)
&= \mathbb E_{\mathbf{Z}_T\mid x}\!\left[\log q(z_{T}\mid x)-\log q_T(z_T)\right].
\end{align*}
Since
$$
q\left(z_T \mid x\right)=\left(2 \pi \sigma_{q, T}^2\right)^{-d / 2} \exp \left(-\frac{\left|z_T-\sqrt{\alpha_T} x\right|^2}{2 \sigma_{q, T}^2}\right),
$$
and
$$
q_T\left(z_T\right)=\left(2 \pi \sigma_{q, T}^2\right)^{-d / 2} \int q_0\left(x^{\prime}\right) \exp \left(-\frac{\left|z_T-\sqrt{\alpha_T} x^{\prime}\right|^2}{2 \sigma_{q, T}^2}\right) \mathrm{d} x^{\prime},
$$
yielding
\begin{align*}
\log\frac{q(z_{T}\mid x)}{q_T(z_{T})}
&=-\frac{1}{2\sigma^2}\,\|z_{T}-\sqrt{\alpha_{T}}x\|_2^2
-\log\!\int q_0(x')\exp\!\left(-\frac{\|z_{T}-\sqrt{\alpha_{T}}x'\|_2^2}{2(1-\alpha_T)}\right)\mathrm dx'.
\end{align*}
Applying Jensen's inequality,
\begin{align*}
D_{{KL}}\!\bigl(q(z_T\mid x)\,\|\,q_T(z_T)\bigr)
&\leq -\frac{d}{2}
+\frac{1}{2\sigma^2}\,\mathbb E_{\mathbf{Z}_T\mid x}\mathbb E_{\mathbf{X}'}\bigl\|z_T-\sqrt{\alpha_T}x'\bigr\|_2^2 \\
&= -\frac{d}{2}
+\frac{1}{2\sigma^2}\Bigl(\sigma^2 d+\alpha_T\mathbb E_{\mathbf{X}'}\|x-x'\|_2^2\Bigr) \\
&=\frac{\alpha_T}{2(1-\alpha_T)}\;\mathbb E_{\mathbf{X}'\sim q_0}\bigl[\|x-x'\|_2^2\bigr].
\end{align*}
By Assumption \ref{assump:target1}, the support is on $\Omega=[-1,1]^d$, then $\|\mathbf{X}\|_2\leq \sqrt d$, $\mathrm{tr}\,\Sigma\leq d$, and $\|x-\mathbb E[\mathbf{X}]\|_2\leq 2\sqrt d$, hence
 \begin{align*}
     D_{{KL}}\left(q\left(z_T \mid x\right) \| q_T\left(z_T\right)\right) \leq \frac{5}{2} d \cdot \frac{\alpha_T}{1-\alpha_T} .
 \end{align*}

\subsection{Distance between $\tilde p(z_{t-1} \mid z_t)$ and $q(z_{t-1} \mid z_t, x)$}
\label{sec:proof of p/q}
\subsubsection{Proof of Proposition \ref{prop:z_tmin}}
As both $f_1$ and $f_2$ in Eqs.\ref{eq:ddpmf1} and \ref{eq:ddpmf2} involve Gaussian kernels, we aim to demonstrate that functions of the form, $$f( { z}_t)\;:=\;
\int_{\Omega} q_{0}(x)
\exp(-\frac{\|\sqrt{\alpha_t}\,x-z_t\|_{2}^{2}}        {2(1-\alpha_t)})\,d x,
\qquad { z}_t\in\mathbb R^{d}$$ can be approximated via Proposition \ref{prop:kernel}.
  To proceed, define $\mu({z_t},{x},\alpha_t)\coloneq \frac{\|\sqrt{\alpha_t}\,x-z_t\|_{2}^{2}}{2(1-\alpha_t)}$. Assuming $\|{x}\|_2\leq 1$ and $\left|{z}_t\right| \leq(\sqrt{2\xi}+\sqrt{d}+1)\sqrt{(1-{\alpha}_t)}$,  it follows that $\mu({z_t},{x},\alpha_t)\leq\xi$. By applying Lemma~\ref{lemma:convex-hull} to $f( {z_t})$ with  $\psi(x)= \exp(-x)$, we get
  \begin{align*}
&f\left(z_t\right)=\int_{\Omega}\left(\int_0^{\xi}\left[(\mu-t)_{+} \psi^{\prime \prime}(t)+(-\mu-t)_{+} \psi^{\prime \prime}(-t)\right] d t-\mu+1\right) q_0(x) d x.
\end{align*}
Note that the integral part after normalization,
\begin{align*}
 I(\mu)=\frac{1}{\|\psi''\|_{1}}\int_{\Omega}\int_0^{\xi}\left[(\mu-t)_{+} \psi^{\prime \prime}(t)+(-\mu-t)_{+} \psi^{\prime \prime}(-t)\right] d tq_0(x)dx,
\end{align*} belongs to the space $\mathcal S$. Using Proposition \ref{prop:kernel}, there exists $\phi( {x}) \in \CN\CN(\CW,1,\mathcal{M})$ with width $\CW=O(n)$ and parameter norm $\mathcal{M}$ such that
\begin{align*}
        |{f( {z}_t)-\phi}| \precsim  n^{-\frac{1}{2}-\frac{3}{2d}}. 
\end{align*} 
 In particular, for $f_1$ and $f_2$, there exist shallow neural networks $\phi_{f_1}\in\CN\CN(\CW_1,1,\mathcal{M}_1)$ and $\phi_{f_2}\in\CN\CN(\CW_2,1,\mathcal{M}_2)$ with widths $\CW_1=\mathcal O(n_1)$ and $\CW_2=\mathcal O(n_2)$,
and parameter norms $\mathcal{M}_1>0$ and $\mathcal{M}_2>0$, such that,
\begin{align*}
 &\left|f_1-\phi_{f_1}\right| \precsim n_1^{-\frac{1}{2}-\frac{3}{2d}},\\
 &\left|f_2-\phi_{f_2}\right| \precsim n_2^{-\frac{1}{2}-\frac{3}{2d}}.
\end{align*}
Basing on Theorem \ref{thm:f1/f2}, the ratio network is denoted component-wise by
$$\phi(x):=\left(\phi_1\left(\phi_{f_2}(x), \phi_{f_1, 1}(x)\right), \ldots, \phi_d\left(\phi_{f_2}(x), \phi_{f_1, d}(x)\right)\right)^{\top}.$$
As $e^{-\xi} \precsim f_2$ when $\left|z_t\right|_2 \leq (\sqrt{2\xi}+\sqrt{d}+1)\sqrt{(1-{\alpha}_t)}$, then we have,
\begin{equation*}
\left|\frac{f_{1, i}}{f_2}-\frac{\widetilde{f}_{1, i}}{\widetilde{f}_2}\right| \precsim \frac{1}{\mathcal{M}_2}\left|f_{1, i}-\widetilde{f}_{1, i}\right|+\frac{e^{\xi}}{\mathcal{M}_2}\left|f_2-\widetilde{f}_2\right|,
\end{equation*}
for $i=1,\cdots,d$. Using the approximation error bounds of $\phi_{f_1}$ and $\phi_{f_2}$ above yields
\begin{align*}
    \left|\frac{f_1}{f_2}-\phi\right| \precsim \frac{1}{\mathcal{M}_2} n_1^{-\frac{1}{2}-\frac{3}{2 d}}+\frac{e^{\xi}}{\mathcal{M}_2} n_2^{-\frac{1}{2}-\frac{3}{2 d}}.
\end{align*}
Taking $n=\min\{n_1,n_2\}$ gives
\begin{align*}
    \left|\frac{f_1}{f_2}-\phi\right| \precsim \frac{e^{\xi}}{\mathcal{M}_2} n^{-\frac{1}{2}-\frac{3}{2d}} .
\end{align*}
The division sub-network is applied to each output coordinate, adding depth $6$, width at most $9d$, and at most $71d$ number of parameters. Then, there exists a SignReLU network $\CN\CN(\CW,\CL,\mathcal{M})$ with depth $\CL=7$, width $\CW= O(n+9d)$ and bounded parameter norm $\mathcal{M}>0$, such that for $f_\rho=\frac{f_1}{f_2}$ and $\tilde \phi_t \in \CN\CN(\CW,\CL,\mathcal{M})$,
\begin{align*}
    \Bigl|f_\rho-\tilde \phi_t\Bigr|
\precsim C(C_{q_0}, \alpha_t)\frac{e^{\xi}}{\mathcal{M}} n^{-\frac{1}{2}-\frac{3}{2d}},
   \end{align*}
Since
\begin{align*}
    \mathbb{P}\left(\left|\mathbf{Z}_t\right| \leq(\sqrt{2 \xi}+\sqrt{d}+1) \sqrt{1-\alpha_t}\right)=1-e^{-\xi},
\end{align*}
Then, over the high-probability region of $z_t$, we obtain the following bound:
\begin{align*}
    \int_{\left|z_t\right| \leq(\sqrt{2 \xi}+\sqrt{d}+1) \sqrt{1-\alpha_t}}\left|f_\rho\left(z_t\right)-\tilde{\phi}_t\left(z_t\right)\right|^2 d z_t \precsim C\left(C_{q_0}, \alpha_t\right) \cdot \frac{e^{2\xi}}{\mathcal{M}^2} \cdot n^{-1-\frac{3}{d}},
\end{align*}
which completes the proof.
\subsubsection{Proof of Proposition \ref{prop:z_tmax}}
We first analyze the behavior of $f_1$ and $f_2$ in Eqs.\ref{eq:ddpmf1} and \ref{eq:ddpmf2} for inputs satisfying the condition $\|z_t\|_2\geq \sqrt{1 - \alpha_t}(\sqrt{2\xi} + \sqrt{d} + 1)$.
For $f_1(z_t)\in \mathbb R^d$,
\begin{align}
&\precsim\left|\int\left(\frac{1-\alpha_{t-1}}{1-\alpha_t} \sqrt{1-\beta_t} z_t+\frac{\sqrt{\alpha_{t-1}} \beta_t}{1-\alpha_t}\right) \cdot q\left(z_t \mid x\right)d x\right|\notag\\
&{=}\left|\frac{1}{\sqrt{\alpha_t}}\int\left(\frac{1-\alpha_{t-1}}{1-\alpha_t} \sqrt{1-\beta_t} z_t+\frac{\sqrt{\alpha_{t-1}} \beta_t}{1-\alpha_t}\right) \cdot\frac{1}{\sqrt{2(1-\alpha_t)\pi} } e^{-\frac{\|\sqrt{\alpha_t}x- z_t\|_2^2}{2 (1-\alpha_t)}}dx\right|\notag\\
&\precsim \frac{d}{\sqrt{2 \pi\alpha_t(1-\alpha_t)}}\|z_t\|_2e^{-\xi},\label{eq:z_max1}
\end{align}
where the last step due to the bound $\left\|\sqrt{\alpha_t} x-z_t\right\| \geq \sqrt{2 \xi\left(1-\alpha_t\right)}$ the condition $\left|z_t\right|$ and the constraint $\|x\|_2 \leq 1$. With the similar argument, for $f_2\left(z_t\right)\in \mathbb R$,
\begin{align}
 f_2(z_t)&\geq \frac{1}{C_{q_0}\sqrt{\alpha_t}}\int \frac{1}{\sqrt{2 \pi(1-\alpha_t)} } e^{-\frac{\|\sqrt{\alpha_t}x- z_t\|_2^2}{2 (1-\alpha_t)}}dx\notag\\
    &\asymp\frac{1}{2C_{q_0}\sqrt{\alpha_t}}e^{-\xi}\label{eq:z_max2}.
\end{align}
Then, dominated by the tail of the Gaussian density centered at $z_t$ that the probability mass decays exponentially,
we have that
\begin{align*}
\int_{\left|z_t\right| > (\sqrt{2\xi}+\sqrt{d}+1)\sqrt{(1-{\alpha}_t)}}|\frac{f_1}{f_2}|^2dz_t&\precsim \int_{\left|z_t\right| > (\sqrt{2\xi}+\sqrt{d}+1)\sqrt{(1-{\alpha}_t)}}\frac{2}{\pi\sigma^2_{q,t}}z^2_tdz_t\\
&=\frac{2C_{q_0}d^2(2\xi+1)}{\pi}\cdot e^{-\xi}.
\end{align*}
On the other hand, the output of the network $\tilde{\phi}_t$ is assumed to be uniformly bounded by a constant $C\left(\alpha_t, d\right)$, which is a reasonable assumption to control extreme outlier cases. Specifically, for any input satisfying
$$
\left|z_t\right|>(\sqrt{2 \xi}+\sqrt{d}+1) \sqrt{1-\alpha_t},
$$
we have
$$
\left|\tilde{\phi}_t\left(z_t\right)\right| \leq C\left(\alpha_t, d\right) .
$$
This yields the following bound on the tail integral:
$$
\int_{\left|z_t\right|>(\sqrt{2 \xi}+\sqrt{d}+1) \sqrt{1-\alpha_t}}\left|\frac{f_1\left(z_t\right)}{f_2\left(z_t\right)}-\tilde{\phi}_t\left(z_t\right)\right|^2 d z_t \precsim\left(4 \xi+C^2\left(\alpha_t, d\right)\right) \cdot e^{-\xi},
$$
and thus completes the proof.

\subsubsection{Discussion on Part(III)}
\label{sec:Hyper-parameters}
The Part(III) discusses the bounding of
\begin{align*}
    E_{\mathbf{Z}_t} E_{\mathbf{X} \mid \mathbf{Z}_t}\left[ \frac{1}{2 \sigma_{p,t}^2}\left(\left|\frac{1-\alpha_{t-1}}{1-\alpha_t} \sqrt{1-\beta_t} z_t+\frac{\sqrt{\alpha_{t-1}} \beta_t}{1-\alpha_t} x-f_\rho\left(z_t\right)\right|^2\right)\right].
\end{align*}
To make the computation simple and clear, we define $\mu\left(z_t\right):=E\left[x \mid z_t\right]$ and $\Sigma\left(z_t\right):=\operatorname{Var}\left[x \mid z_t\right]$, such that
\begin{align*}
 g\left(z_t\right):=&E_{\mathbf{X} \mid \mathbf{Z}_t}\left[\left|A z_t+B x-f_\rho\left(z_t\right)\right|^2\right] \\
 =&\left|A z_t+B \mu\left(z_t\right)-f_\rho\left(z_t\right)\right|^2+B^2 \operatorname{tr}(\Sigma\left(z_t\right)).
\end{align*}
Let ${G}(x):={q}\left(z_{{t}} \mid x\right)$ be the Gaussian kernel. For $q_0(x) \in\left[C_{q_0}^{-1}, C_{q_0}\right]$, we have
\begin{align*}
&C^{-2}_{q_0} \mu_G\left(z_t\right) \leq \mu\left(z_t\right) \leq  C^2_{q_0} \mu_G\left(z_t\right), \\
&C^{-2}_{q_0} \Sigma_G\left(z_t\right) \leq \Sigma\left(z_t\right) \leq C^2_{q_0} \Sigma_G\left(z_t\right) .
\end{align*}
Based on the definition of diffusion model, $\mu_G\left(z_t\right)=\sqrt{\alpha_t} x$ and $\Sigma_G\left(z_t\right)=\left(1-\alpha_t\right)\mathbf{I}$.
\begin{align*}
    &\mu_{\text {low}}\left(z_t\right)=C^{-2}_{q_0}\sqrt{\alpha_t} x, \ \mu_{\text{high}}\left(z_t\right)=C^2_{q_0} \sqrt{\alpha_t} x,\\
    & \Sigma_\text {low}\left(z_t\right)=C^{-2}_{q_0}\left(1-\alpha_t\right)\mathbf{I}, \ \Sigma_\text {high}\left(z_t\right)= C^2_{q_0}\left(1-\alpha_t\right)\mathbf{I}.
\end{align*}
The lower and upper bound holds for $g\left(z_t\right)$.
\begin{align*}
    &g_\text{low}\left(z_t\right)=| A z_t+B \mu_\text{low}\left(z_t\right)-f_\rho\left(z_t\right)|^2+B^2\operatorname{tr}\Sigma_\text{low}\left(z_t\right),\\
    &g_\text{high}\left(z_t\right)=| A z_t+B \mu_\text{high}\left(z_t\right)-f_\rho\left(z_t\right)|^2+B^2\operatorname{tr}\Sigma_\text{high}\left(z_t\right).
\end{align*}
As $$E_{\mathbf{Z}_t}\left[g_\text{low}\left(z_t\right)\right] \leq E_{\mathbf{Z}_t}\left[g\left(z_t\right)\right] \leq E_{\mathbf{Z}_t}\left[g_\text{high}\left(z_t\right)\right],$$ then
\begin{align*}
     &L_t^{-}=\frac{1}{2}\sigma_{p,t}^{-2} E_{\mathbf{Z}_t}\left[\left|A z_t+B C_{q_0}^{-2}\sqrt{\alpha_t}z_t-f_\rho\left(z_t\right)\right|^2\right] +2d(1-\alpha_t)B^2C_{q_0}^{-2} \sigma_{p,t}^{-2},\\
     &L_t^{+}=\frac{1}{2}\sigma_{p,t}^{-2} E_{\mathbf{Z}_t}\left[\left|A z_t+B C_{q_0}^{2}\sqrt{\alpha_t}z_t-f_\rho\left(z_t\right)\right|^2\right] +2d(1-\alpha_t)B^2C_{q_0}^{2} \sigma_{p,t}^{-2},
\end{align*}
which implies,
\begin{align*}
    L_t^{-} \leq L_t \leq L_t^{+}.
\end{align*}
Bounding $L_t^{+}$ and $L_t^{-}$ thus reduces to bounding the behavior of the function $f_\rho(z_t) = \frac{f_1(z_t)}{f_2(z_t)}$. In particular, the worst-case occurs when $f_2(z_t) \to 0$ and $|z_t| > (\sqrt{2\xi} + \sqrt{d} + 1) \sqrt{1 - \alpha_t}$.

From the tail bound analysis in Eq.\ref{eq:z_max1} and Eq.\ref{eq:z_max2}, we conclude that $f_\rho$ is uniformly bounded by a constant $C_\rho(\alpha_t, \sigma_{p,t}, C_{q_0})$,
which completes the proof of the stated bounds.
\subsection{Proof of Theorem \ref{thm:approx}}
\label{sec:proof of approx}
Combing the results,
\begin{align}
    \label{eq:approx_proof}
         &D\left(q_0 \| \tilde{p}_0\right)\precsim C(C_{q_0},\alpha_t,\sigma_{p,t}) + T\underset{t}{\sup}(B_{Approx1} + B_{Approx2}),
    \end{align}
    where  
   \begin{align*}
       B_{Approx1} &= C\left(C_{q_0}, \alpha_t\right) \cdot \frac{e^{2\xi}}{\mathcal{M}^2} \cdot n^{-1-\frac{3}{d}},\\
       B_{Approx2} &= O((4\xi+C^2(\alpha_t,d))\cdot e^{-\xi}).
   \end{align*}

Regarding $B_{\text{Approx2}}$ as the parameter $\xi$ increases, the exponential decay term $e^{-\xi}$ dominates more rapidly, causing the overall expression of order $O\left(\xi e^{-\xi}\right)$ to eventually decrease. Notably, this expression attains its maximum at $\xi=1$.

Following the analysis of $B_{\text{Approx1}}$, let us take $\xi=\log \mathcal{M}$, yielding
$$
B_{\text{Approx1}}=C\left(C_{q_0}, \alpha_t\right) \cdot \frac{e^{2 \xi}}{\mathcal{M}^2} \cdot n^{-1-\frac{3}{d}} \asymp n^{-1-\frac{3}{d}},
$$
which depends only on the structure of the network. Since $\xi=\log \mathcal{M} \gg 1$, we have
$$
B_{\text{Approx1}}=O\left(\frac{\log \mathcal{M}}{\mathcal{M}}\right),
$$
which converges to zero as $\mathcal{M} \rightarrow \infty$.
Therefore, the KL divergence between the data distribution $q_0$ and the learned distribution $\tilde{p}_0$ admits the following upper bound:
$$
D\left(q_0 \| \tilde{p}_0\right) \precsim C\left(C_{q_0}, \alpha_t, \sigma_{p, t}\right)+T \cdot \sup _t\left(n^{-1-\frac{3}{d}}+\frac{\log \mathcal{M}}{\mathcal{M}}\right) .
$$

\section{Estimation Error}
\label{sec:proof of esti}
\subsection{Statistical error analysis}
\label{sec:proof of Statistical error}
\subsubsection{Proof of Proposition \ref{prop:logbound}}
An analysis of the boundedness of $\left\|\log \tilde{p}_0(x)\right\|_{\infty}$ is first conducted, as it is a prerequisite for the ideal reversal of the forward process. Beginning from the definition,
\begin{align*}
\left\|\log \tilde{p}_0(x)\right\|_{\infty}=\left\|\log \int q_T\left(z_T\right) \cdot \tilde{p}\left(x \mid z_1\right) \prod_{t=2}^T p\left(z_{t-1} \mid z_t\right) d z_1 \ldots d z_T\right\|_{\infty},
\end{align*}
the following bound is derived,
\begin{align*}
         \|\log \tilde p_0(x)\|_\infty \leq\| \log \prod_{t=1}^T\left(2 \pi \sigma_{p,t}^2\right)^{-\frac{d}{2}}\int \exp \left(-\sum^T_{t=1}\frac{\left|\phi_t(z_t)\right|^2}{ \sigma_{p,t}^2}-\sum^T_{t=2} \frac{\left|z_{t}\right|^2}{\sigma_{p,t}^2}-\frac{\left|x\right|^2}{\sigma_{p,1}^2}\right)d z_1, \ldots, z_T\|_\infty
\end{align*}
Under Assumption \ref{assump:target1}, which states $\mathbb{P}\left(\|X\|_2 \leq 1 \mid X \sim q_0\right)=1$, and given that for sufficiently large $\xi$, the inequality $\left|z_t\right| \leq (\sqrt{2 \xi}+\sqrt{d}+1) \sqrt{1-\alpha_t}$ holds with probability at least $1-e^{-\xi}$ for all $t \in T$, it follows that
\begin{align*}
         \|\log \tilde p_0(x)\|_\infty 
        \precsim \left|\log\prod_{t=1}^T\left(2 \pi \sigma_{p,t}^2\right)^{-\frac{d}{2}}\right|+T(1-\alpha_t)(\sqrt{2\xi}+1)^2+\sup_{\phi_t}\left|\log \left\{\exp\left(-\sum^T_{t=1}\frac{\left|\phi_t\right|^2}{\sigma_{p,t}^2}\right)\right\}\right|.
\end{align*}
Given that each function $\phi_t$ belongs to a neural network class $\mathcal{N} \mathcal{N}(\mathcal{W}, \mathcal{L}, \mathcal{M})$ with its parameter norm bounded by $\mathcal{M}$, the final bound is obtained,
\begin{align*}
        \|\log \tilde p_0(x)\|_\infty\precsim T\left(\mathcal{M}^2+(\sqrt{2\xi}+\sqrt{d}+1)^2\right) +C(d,\sigma_{p,t}, T).
    \end{align*}

We next analyze an upper bound for $\|\log \hat{p}_0(\hat{z}_0)\|_{\infty}$. As a minor modification, if $|\hat{z}_0| \geq 2$, we reset $\hat{z}_0$ to $0$. This adjustment does not increase the estimation error, since $|x| \leq 1$ by Assumption \ref{assump:target1}.
\begin{prop}[Log-density bounds for the backward process]
\label{prop:logbound-fixed}

As $\mathbf{\hat Z}_T \sim\mathcal{N}(0,I_d)$, we fix $\xi>0$ and define the radius recursively of $\hat z_t$ at each time t by for $t=T,T-1,\dots,2,1$, 
\begin{align*}
R_T(\xi):=\sqrt{d}+\sqrt{2\xi}, 
\quad
R_{t-1}(\xi):= a_t\, R_t(\xi) + b_t + \sigma_{p,t}(\sqrt{d}+\sqrt{2\xi}), 
\end{align*}
Suppose we implement $\phi_t\in\CN\CN(\CW,\CL,\mathcal{M})$, then the operator norm $a_t$ and $b_t$ is bounded by $\mathcal{M}$.
We have
\begin{enumerate}
\item[(i)] (Mass on a high-probability ball) 
\begin{align*}
\mathbb{P}\left(\left|\hat{\mathbf{Z}}_1\right| \leq R_1(\xi)\right) \geq 1-(T+1) e^{-\xi} .
\end{align*}
\item[(ii)] (Pointwise upper bound) For all $\hat z_0\in[-2, 2]^d$,
\begin{equation}
\label{eq:upper}
\log \hat{p}_0\left(\hat{z}_0\right) \leq-\frac{d}{2} \log \left(2 \pi \sigma_{p, 1}^2\right).
\end{equation}
\item[(iii)] (Pointwise lower bound) For all $\hat z_0\in[-2,2]^d$,
\begin{equation}
\label{eq:lower}
\log \hat{p}_0\left(\hat{z}_0\right) \geq-\frac{d}{2} \log \left(2 \pi \sigma_{p, 1}^2\right)+\log \left(1-(T+1) e^{-\xi}\right)-\frac{\left(2\sqrt{d}+a_1 R_1(\xi)+b_1\right)^2}{2 \sigma_{p, 1}^2} .
\end{equation}
\item[(iv)] ($L^\infty$ bound) Combining Eq.\ref{eq:upper}--Eq.\ref{eq:lower}, we obtain
\begin{equation*}
\left\|\log \hat{p}_0\left(\hat{z}_0\right)\right\|_{\infty} \precsim \left(\sqrt{d}+\sqrt{2 \xi}\right)^2T^4\mathcal{M}^2+C(d,\sigma_{p,t}, T).
\end{equation*}
\end{enumerate}
\end{prop}

\begin{proof}
(i) Define events
\begin{equation*}
E_T:=\left\{\left|\hat{\mathbf{Z}}_T\right| \leq \sqrt{d}+\sqrt{2 \xi}\right\}, \quad F_t:=\left\{\left|\varepsilon_t\right| \leq \sqrt{d}+\sqrt{2 \xi}\right\}, \quad t=T, \ldots, 1 .
\end{equation*}
We have that $\mathbb{P}(E_T) \geq 1 - e^{-\xi}$ and $\mathbb{P}(F_t) \geq 1 - e^{-\xi}$ hold for each $t$. Suppose that both events $E_T$ and $F_t$ hold for all $t$. Given the relation $\hat{z}_{t-1} = \phi_t(\hat{z}_t) + \sigma_{p,t} \varepsilon_t$, it follows that
$$
\left|z_{t-1}\right| \leq\left|\phi_t\left(z_t\right)\right|+\sigma_{p, t}\left|\varepsilon_t\right| \leq a_t\left|z_t\right|+b_t+\sigma_{p, t}(\sqrt{d}+\sqrt{2 \xi}) .
$$
Inductively with $R_T(\xi)=\sqrt{d}+\sqrt{2\xi}$, this yields $|\mathbf{\hat Z}_t|\leq R_t(\xi)$ for all $t$, hence $|\mathbf{\hat Z}_1|\leq R_1(\xi)$. By a union bound over the $T+1$ events $\{E_T,F_T,\dots,F_1\}$,
$$
\mathbb{P}\left(\left|\hat{\mathbf{Z}}_1\right| \leq R_1(\xi)\right) \geq 1-(T+1) e^{-\xi}=1-\delta_T(\xi) .
$$

(ii) Since $\hat p_0(\hat z_0)$ is a integral of Gaussians with common covariance $\sigma_{p,1}^2 I_d$,
$$
\hat{p}_0\left(\hat{z}_0\right)=\int \hat{p}_1\left(\hat{z}_1\right) \hat{p}\left(\hat{z}_0 \mid \hat{z}_1\right) d \hat{z}_1 \leq \sup _{\hat{z}_1} \hat{p}\left(\hat{z}_0 \mid \hat{z}_1\right)=\left(2 \pi \sigma_{p, 1}^2\right)^{-\frac{d}{2}}.
$$
Taking logarithms gives Eq.\ref{eq:upper}.

(iii) Fix $\hat z_0\in[-2,2]^d$, hence $|\hat z_0|\leq 2\sqrt{d}$. Let $E_1(\xi):=\{\hat z_1:|\hat z_1|\leq R_1(\xi)\}$. Then
\begin{align*}
\hat{p}_0\left(\hat{z}_0\right) & =\int \hat{p}_1\left(\hat{z}_1\right) \hat{p}\left(\hat{z}_0 \mid \hat{z}_1\right) d \hat{z}_1 \\
&\geq \int_{E_1(\xi)} p_1\left(z_1\right) p\left(\hat{z}_0 \mid \hat{z}_1\right) d \hat{z}_1 \\
& \geq \mathbb{P}\left(E_1(\xi)\right) \cdot\left(2 \pi \sigma_{p, 1}^2\right)^{-\frac{d}{2}} \cdot \exp \left(-\frac{\sup _{\hat{z}_1 \in E_1(\xi)}\left|\hat{z}_0-\phi_1\left(\hat{z}_1\right)\right|^2}{2 \sigma_{p, 1}^2}\right).
\end{align*}
By the triangle inequality, for $\hat z_1\in E_1(\xi)$,
$$
|\hat z_0-\phi_1(\hat z_1)| \leq |\hat z_0| + |\phi_1(\hat z_1)| \leq 2\sqrt{d} + a_1 R_1(\xi)+b_1.
$$
Using (i), $\mathbb{P}\bigl(E_1(\xi)\bigr)\ge 1-\delta_T(\xi)$. Taking logarithms yields Eq.\ref{eq:lower}.

(iv) For the network $\CN\CN(\CW,\CL,\mathcal{M})$, the parameters are bounded by $\mathcal{M}$ and follows that $a_t$ and $b_t$ bounded by $\mathcal{M}$. Suppose for each step t, there exist $c_0\geq 1$, such that
$$\mathcal{M}_t\leq \mathcal{M}^{\frac{1}{(t+c_0)\log \mathcal{M}}}.$$
In this way, $a_t \leq \mathcal{M}_t$, products satisfy
\begin{align*}
\prod_{k=2}^{T} a_k \leq \exp(\sum_{t=2}^{T} \frac{1}{(t+c_0)\log \mathcal{M}}\log \mathcal{M})= C(c_0)(T+c_0).
\end{align*}
Expanding the backward process gives
\begin{align*}
    R_1(\xi)
\;\le\; \Big(\prod_{k=2}^{T} a_k\Big) R_T(\xi)
\;+\; \sum_{j=2}^{T} \Big(\prod_{k=2}^{j-1} a_k\Big)\, \bigl(1 + \sigma_{p,j} R_T(\xi)\bigr).
\end{align*}
Finally, plug these $R_1(\xi)$ bounds and $a_1\leq \mathcal{M}$ into Eq.\ref{eq:upper}--Eq.\ref{eq:lower} and take the supremum over $\hat z_0\in[-2,2]^d$ to obtain the displayed bound.
\end{proof}
For a neural network $\NN(\CW,\CL,\mathcal{M})$ of depth $\CL = 7$, width $\CW = O(n+9d)$ and parameter norm $\mathcal{M} >0$, we set $\xi \asymp \log \mathcal{M}$ according to Theorem \ref{thm:approx}. This completes the proof.

\subsubsection{Proof of Proposition \ref{prop:staterror}}
To derive an upper bound for the statistical error, a standard analytical approach is employed in Lemma \ref{lemma:statistical error}. Specifically, the upper bound is first expressed in terms of the Rademacher complexity. Subsequently, the Rademacher complexity is controlled using covering numbers. 

\begin{lemma}[Statistical error bounding \cite{huang2022error}]
\label{lemma:statistical error}

    Suppose $\sup _{h \in \mathcal{H}}\|h\|_{\infty} \leq B$, then we can bound $E[d_\mathcal{H} (\mu,\hat\mu)]$ with $d_\mathcal{H} (\mu,\hat\mu)=\sup _{h \in \mathcal{H}} E_{x \sim \mu}[h(x)]-\frac{1}{n} \sum_{i=1}^n h\left(x_i\right)$,
    \begin{equation*}
    \begin{aligned}
E\left[d_{\mathcal{H}}\left(\mu, \widehat{\mu}\right)\right] \leq 2 E \inf _{0<\eta<B / 2}\left(4 \eta+\frac{12}{\sqrt{n}} \int_\eta^{B / 2} \sqrt{\log \mathcal{N}\left(\varepsilon, \mathcal{H}_{\mathbf{\hat x}}, \|\cdot\|_{\infty}\right)} d \epsilon\right),
\end{aligned}
\end{equation*}
where we denote $\mathcal{\hat H}=\left\{\left(h\left(\mathbf{x}_1\right), \ldots, h\left(\mathbf{x}_n\right)\right): h \in \mathcal{H}\right\}$ for any i.i.d. samples $\left\{\mathbf{x}_i\right\}_{i=1}^n$ from $\mu$ and $\mathcal{N}\left(\epsilon, \mathcal{\hat H},\|\cdot\|_{\infty}\right)$ is the $\epsilon$-covering number of $\mathcal{\hat H} \subseteq \mathbb{R}^d$ with respect to the $\|\cdot\|_{\infty}$ distance.
\end{lemma}
Let $\mathcal{H}$ denote the function class corresponding to $h(x) \coloneq \log \tilde{p}_0(x)$, and let $\mathcal{G}$ denote the function class corresponding to $g(\hat{z}_0) \coloneq \log \hat{p}_0(\hat{z}_0)$. Since $h(x)$ is bounded by $B_{\tilde{p}}$ and $g(\hat{z}_0)$ is bounded by $B_{\hat{p}}$, Lemma \ref{lemma:statistical error} can be applied to obtain that, with probability at least $1 - 2\delta$,
    \begin{align*}
		&D_{KL}\left(q_0(x) \| \hat{p}_0(\hat z_0)\right)-D_{KL}\left(q_0 (x)\| \tilde{p}_0(x)\right)\\ 
		&\leq E_{\mathbf{X}\sim q_0} \log \tilde{p}_0(x) - \frac{1}{m}\sum_{i=1}^m \log \tilde{p}_0(x_i)+\frac{1}{m}\sum_{i=1}^m \log \hat{p}_0(\hat z_0)-E_{\mathbf{\hat Z}_0\sim \hat p_0} \log \hat p_0(\hat z_0).
    \end{align*}
This inequality follows from the observation that, for part ($\mathrm{III}$) in Eq.\ref{eq:esti}, Proposition \ref{prop:logbound} implies
$$\frac{1}{m}\sum_{i=1}^m \log \tilde{p}_0(x_i)-\frac{1}{m}\sum_{i=1}^m \log \hat{p}_0(\hat z_{0,i}) \leq 0.$$
Hence, we obtain the bound
    \begin{align*}
           D_{KL}\left(q_0(x) \| \hat{p}_0(\hat z_0)\right)-D_{KL}\left(q_0 (x)\| \tilde{p}_0(x)\right)\leq B_{\tilde p} \sqrt{\frac{2\log \frac{1}{\eta_1}}{ m}}+ \mathcal{R}\left(\mathcal{H}\right)+B_{\hat p} \sqrt{\frac{2\log \frac{1}{\eta_2}}{ m}}+ \mathcal{R}\left(\mathcal{G}\right),
    \end{align*}
    where
        \begin{align*}
\mathcal{R}\left(\mathcal{H}\right)&=E_{\mathbf{X}, \epsilon} \sup _{h \in \mathcal{H}}\left|\frac{1}{m} \sum_{i=1}^{m} \epsilon_i h\left( x_i\right)\right|\\
        &\leq E\inf _{0<\eta_1< B_{\tilde p}}\left(4 \eta_1+\frac{12}{\sqrt{m}} \int_{\eta_1}^{ B_{\tilde p}} \sqrt{\log \mathcal{N}\left(\varepsilon, \mathcal{\hat H}, \|\cdot\|_{\infty}\right)} d \epsilon\right),\\
        \mathcal{R}\left(\mathcal{G}\right)&=E_{\mathbf{\hat Z}_0, \epsilon} \sup _{g \in \mathcal{G}}\left|\frac{1}{m} \sum_{i=1}^{m} \epsilon_i g\left( \hat z_{0,i}\right)\right|\\
        &\leq E \inf _{0<\eta< B_{loss}}\left(4 \eta_2+\frac{12}{\sqrt{m}} \int_{\eta_2}^{ B_{\hat p}} \sqrt{\log \mathcal{N}\left(\varepsilon, \mathcal{\hat G}, \|\cdot\|_{\infty}\right)} d \epsilon\right).
    \end{align*}
\subsection{Proof of Theorem \ref{thm:esti}}
The proof relies on the bounding of covering number $\mathcal{N}\left(\varepsilon, \mathcal{\hat H}, \|\cdot\|_{\infty}\right)$ and $\mathcal{N}\left(\varepsilon, \mathcal{\hat G}, \|\cdot\|_{\infty}\right)$.
\begin{prop}
\label{prop:cover_num}
    Let $\mathcal{H} := \mathcal{N}\mathcal{N}(\mathcal{W},\mathcal{L},\mathcal{M})$ be the class of functions defined by a $\operatorname{SignReLU}$ neural network on $[-1,1]^d$, 
    $$\mathcal{H}=\left\{h_1: h_1(x)= h^{[L]}, h^{[j]}=\sigma\left(A_j^{\top} h^{[j-1]}+b_j\right), h^{[0]}=x\right\},$$
    with the parameter constraint
$$
\begin{aligned}
\Omega= \left\{ A_j \in \mathbb{R}^{\mathcal{W}_{j-1} \times \mathcal{W}_j}, b_j \in \mathbb{R}^{\mathcal{W}_j}:\right. \left.\max \left\{\left\|A_{j,:, i}\right\|_1,\left\|b_j\right\|_{\infty}\right\} \leqslant \mathcal{M}_j\right\}.
\end{aligned}
$$
Then the covering number of $\mathcal{H}$ can be upper bounded as
$$
\mathcal{N}\left(\epsilon,\mathcal{H}, \|\cdot\|_{\infty}\right) \leq C(\mathcal{W}_j, \mathcal{L})\left(\mathcal{M} / \epsilon\right)^D,
$$
where $\mathcal{M} = \prod_{j=1}^\CL \max_{j}\{\mathcal{M}_j\}$ and $D=\sum_{i=1}^L \mathcal{W}_j \mathcal{W}_{j-1}+\sum_{i=1}^L \mathcal{W}_j\leq \mathcal{W}^2 \mathcal{L}$.
\end{prop}
\begin{proof}
    The parameter constraint denoted as,
\begin{equation}
\label{eq:cn_constraint}
   \begin{aligned}
   &\Omega^\prime_1=\left\{\mathbf{A}_j \in \mathbb{R}^{d_{j-1}\times d_j} : \max_{i,j} \quad\left\|\mathbf{A}_{j, :, i}\right\|_1 \leq \mathcal{M}_j\right\} \subseteq \mathbb{R}^{\sum_{j=1}^\CL d_{j-1} d_j},\\
&\Omega^\prime_2=\left\{\mathbf{b}_j \in \mathbb{R}^{ d_j} : \max_{j} \quad\left\|\mathbf{b}_j\right\|_{\infty} \leq \mathcal{M}_j\right\} \subseteq \mathbb{R}^{\sum_{j=1}^\CL d_j}.
\end{aligned}
\end{equation}
Basing on the tunable parameter $\alpha$ in the activation function SignReLU (see in Eq.\ref{eq:signrelu}), $B=\max\{\alpha, 1\}$. We next analyse the bounding of $\left\|h_1^{[j]}-\tilde h_1^{[j]}\right\|_{\infty}$. For any $j=1, \cdots, \CL$, let $\mathcal{M}^\prime = \max_{j}\{\mathcal{M}_j\}$,
$$
\begin{aligned}
& \left\|h_1^{[j]}\right\|_{\infty}=\left\|\sigma\left(w_k^{\top} h_1^{[j-1]}+\mathbf{b}_k\right)\right\|_{\infty}  \leq B\mathcal{M}^\prime\left\|h_1^{[j-1]}\right\|_{\infty}+B\mathcal{M}^\prime, \\
& \left\|h_1^{[j]}\right\|_{\infty}+\frac{B\mathcal{M}^\prime}{B\mathcal{M}^\prime-1} \leq B\mathcal{M}^\prime\left(\left\|h_1^{[j-1]}\right\|_{\infty}+\frac{B\mathcal{M}^\prime}{B\mathcal{M}^\prime-1}\right), \\
&\left\|h_1^{[j]}\right\|_{\infty} \leqslant {(B\mathcal{M}^\prime)}^j\left(1+\frac{B\mathcal{M}^\prime}{B\mathcal{M}^\prime-1}\right), \\
&\left\|h_1^{[\CL]}\right\|_{\infty} \leq 3{(B\mathcal{M}^\prime)}^\CL .\\
\end{aligned}
$$
Then,
$$
    \begin{aligned}
        \left\|h_1^{[j]}-\tilde h_1^{[j]}\right\|_{\infty}&=\left\|\sigma\left(\mathbf{A}_j^{\top} h_1^{[j-1]}+\mathbf{b}_j\right)-\sigma\left({ \mathbf{\tilde A}}_j^{\top} \tilde h_1^{[j-1]}+\mathbf{\tilde b}_j\right)\right\|_{\infty}\\
        & \leq B\left\|\mathbf{A}_j^{\top} h_1^{[j-1]}+\mathbf{b}_j-\tilde{\mathbf{A}}^{\top}_j \tilde{h_1}^{[j-1]}-\tilde{\mathbf{b}}_j\right\|_{\infty} \\
        &\leq B{\left\|\mathbf{A}_j-\tilde{\mathbf{A}}_j\right\|_1}\left\|h_1^{[j-1]}\right\|_{\infty}+\left\|{\mathbf{\tilde A}_j}\right\|_1\left\|h_1^{[j-1]}-\tilde{h_1}^{[j-1]}\right\|_{\infty}+B\left\|\mathbf{b}_j-\mathbf{\tilde b}_j\right\|_{\infty}\\
        &\leq  3B^j{\mathcal{M}^\prime}^{j-1}\left\|\mathbf{A}_j-\tilde{\mathbf{A}}_j\right\|_1+BD\left\|h_1^{[j-1]}-\tilde{h_1}^{[j-1]}\right\|_{\infty} +\left\|\mathbf{b}_j-\mathbf{\tilde b}_j\right\|_{\infty}\\
        &\leq \CL B^\CL{\mathcal{M}^\prime}^{(\CL-1)}\max_j \left(3\left\|\mathbf{A}_j-\tilde{\mathbf{A}}_j\right\|_1\right. \left.+\left\|\mathbf{b}_j-\mathbf{\tilde b}_j\right\|_{\infty}\right).
    \end{aligned}
    $$
    In this way, for the covering number $ \mathcal{N}\left(\epsilon,\mathcal{H}_1, \|\cdot\|_{\infty}\right)$,
    $$
    \begin{aligned}
        &\mathcal{N}\left(\epsilon,\mathcal{H}_1, \|\cdot\|_{\infty}\right) \leq \mathcal{N}\left(\Omega^\prime_1, \frac{\epsilon}{2 \CL B^\CL{\mathcal{M}^\prime}^\CL},\|\cdot\|_1\right) \cdot \mathcal{N}\left(\Omega^\prime_2, \frac{\epsilon}{2 \CL B^\CL{\mathcal{M}^\prime}^\CL},\|\cdot\|_{\infty}\right).
    \end{aligned}
    $$
    As the parameters are defined in Eq.\ref{eq:cn_constraint}, we can get that,
    $$
\mathcal{N}\left(\epsilon,\mathcal{H}_1, \|\cdot\|_{\infty}\right) \leq C\left(B^\CL{\mathcal{M}^\prime}^{\CL} / \epsilon\right)^D,
$$
where $\mathcal{M} = \prod_{j=1}^\CL \max_{j}\{\mathcal{M}_j\}$ and $D=\sum_{i=1}^\CL \mathcal{W}_j \mathcal{W}_{j-1}+\sum_{i=1}^\CL \mathcal{W}_j\leq \mathcal{W}^2 \mathcal{L}$. Following the construction of the SignReLU network we used in approximation, we consider the tunable parameter $\alpha=1$. 
\end{proof}

Now, we are in the position to prove Theorem \ref{thm:esti}. Let  $\phi_t \in \mathcal{N}\mathcal{N}(\mathcal{W},\mathcal{L},\mathcal{M})$ be SignReLU networks with a depth of $\mathcal{L} = 7$, width $\mathcal{W} = O(n + 9d)$, and parameter norm bounded by $\mathcal{M} = O(n)$. The covering numbers of the function classes $\mathcal{\hat H}$ and $\mathcal{\hat G}$, as defined in Proposition \ref{prop:staterror}, are analyzed using the bounds $B_{\hat{p}}$ and $B_{\tilde{p}}$ from Eq.~\ref{eq:esti}. Based on the framework established in Proposition \ref{prop:cover_num}, the covering numbers of $\mathcal{\hat H}$ and $\mathcal{\hat G}$ can be bounded as follows:
\begin{align*}
	\mathcal{N}\left(\epsilon,\mathcal{\hat H}, \|\cdot\|_{\infty}\right) \leq C\left(B_{\hat p} / \epsilon\right)^{n^2},\\
	\mathcal{N}\left(\epsilon,\mathcal{\hat G}, \|\cdot\|_{\infty}\right) \leq C\left(B_{\tilde p} / \epsilon\right)^{n^2}.
\end{align*}
Apply the upper bound of the covering number to further bounding Proposition \ref{prop:staterror},
     \begin{align*}
       & E_{\mathbf{X}} \log \tilde{p}_0(x) - \frac{1}{m}\sum_{i=1}^m \log \tilde{p}_0(x) \\
        &\leq C E \inf _{0<{\eta_1}< B_{\tilde p} / 2}\left({\eta_1}+\frac{3}{\sqrt{m}} \int_{\eta}^{B_{\tilde p} / 2} \sqrt{\log \mathcal{N}\left(\epsilon, \mathcal{\hat H}, \|\cdot\|_{\infty}\right)} d \epsilon\right)+ B_{\tilde p} \sqrt{\frac{2\log \frac{1}{\delta_1}}{ m}}\\
        &\leq C \inf _{0<{\eta_1}< {B_{\tilde p}} / 2}\left( \eta_1+\sqrt{\frac{C\CW^2\CL}{m}} \int_{\eta}^{B_{\tilde p} / 2} \sqrt{\log (B_{\hat p} / \epsilon)} d \epsilon\right)+  B_{\tilde p} \sqrt{\frac{2\log \frac{1}{\delta_1}}{ m}}\\
        &\leq C \inf _{0<{\eta_1}< {B_{\tilde p}} / 2}\left( \eta_1+B_{\tilde p} \sqrt{\frac{C\CW^2\CL\log (B_{\hat p} / \eta_1)}{m}} \right)+  B_{\tilde p} \sqrt{\frac{2\log \frac{1}{\delta_1}}{ m}}\\
& \leq C B_{\tilde p} (\sqrt{\frac{\CW^2\CL}{m}}+\sqrt{\frac{\log \frac{1}{\delta_1}}{ m}}).
    \end{align*}
Following the same strategy, 
\begin{align*}
	& E_{\mathbf{\hat Z}_0} \log \hat{p}_0(\hat z_0) - \frac{1}{m}\sum_{i=1}^m \log \hat{p}_0(\hat z_0)\leq C B_{\hat p} (\sqrt{\frac{\CW^2\CL}{m}}+\sqrt{\frac{\log \frac{1}{\delta_2}}{ m}}).
\end{align*}
Combining the preceding results yields the following conclusion: for a neural network $\NN(\CW,\CL,\mathcal{M})$ of depth $\CL = 7$, width $\CW = O(n+9d)$ and parameter norm $\mathcal{M}>0$, the estimation error bound holds with probability at least $1 - 2\delta$, as stated in the following,
\begin{align*}
          &D_{KL}\left(q_0(x) \| \hat{p}_0(\hat z_0)\right)-D_{KL}\left(q_0(x) \| \tilde{p}_0(x)\right)\\
		&\precsim T\left(\mathcal{M}^2+\log \mathcal{M}+T^3\mathcal{M}^2\log \mathcal{M}\right)(\sqrt{\frac{7n^2}{m}}+\sqrt{\frac{2\log \frac{1}{\delta}}{ m}}).
\end{align*}

\subsection{Proof of Corollary~\ref{coro:data}}
\label{sec:proof of data size}

Assume the neural network is constructed as in Theorem~\ref{thm:esti}. For each training sample $x_i$, we average the network outputs over the $m_z$ terminal latent draws $\{\hat z_{T,i,j}\}_{j=1}^{m_z}$. The resulting statistical error admits the decomposition
\begin{align}
\label{eq:esti2}
    & D_{K L}\left(q_0(x) \| \hat{p}_0\left(\hat{z}_0\right)\right)-D_{K L}\left(q_0(x) \| \tilde{p}_0(x)\right) \notag\\
= & \underbrace{\mathbb{E}_{x \sim q_0}\left[\log \tilde{p}_0(x)\right]-\frac{1}{m} \sum_{i=1}^{m} \log \tilde{p}_0\left(x_i\right)}_{\mathrm{I}}+\underbrace{\frac{1}{m m_z} \sum_{i=1}^{m} \sum_{j=1}^{m_z} \log \hat{p}_0\left(\hat{z}_{0, i, j}\right)-\mathbb{E}_{\hat{z}_0 \sim \hat{p}_0}\left[\log \hat{p}_0\left(\hat{z}_0\right)\right]}_{\mathrm{II}}\notag \\
& +\underbrace{\frac{1}{m} \sum_{i=1}^{m} \log \tilde{p}_0\left(x_i\right)-\frac{1}{m m_z} \sum_{i=1}^{m} \sum_{j=1}^{m_z} \log \hat{p}_0\left(\hat{z}_{0, i, j}\right)}_{\mathrm{III}}.
\end{align}
The analysis of~\eqref{eq:esti2} follows that of the statistical error in~\eqref{eq:esti}.

Since the network architecture is unchanged, the bound on $\|\log \tilde p_0(x)\|_\infty$ remains the same as in Theorem~\ref{thm:esti}. Next, we account for the effect of processing $m_z$ noise realizations per step in the forward (noise-injection) process. Specifically, at the start of the backward process we draw, for each $i\in\{1,\ldots,m\}$,
\[
\{\hat z_{T,i,j}\}_{j=1}^{m_z}\stackrel{\mathrm{i.i.d.}}{\sim}\CN(0,\mathbf I_d),
\]
and train the backward process using these $m_z$ replicas to generate the corresponding added-noise samples used in the forward process. Under this replication-and-averaging scheme, the essential supremum of the log-density term associated with $\hat p_0$ satisfies
\begin{align*}
B_{\hat p}\coloneq \|\log \hat p_0(\hat z_0)\|_\infty
\ \le\ O\!\left(\frac{1}{m_z}\,T^4\mathcal{M}^2\log \mathcal{M}\right).
\end{align*}

Following the proof of Theorem~\ref{thm:esti}, when the training sample size is $m$, processing $m_z$ added-noise samples at each step $t\in T$ yields
\begin{align*}
&D_{KL}\!\left(q_0(x) \| \hat{p}_0(\hat z_0)\right)-D_{KL}\!\left(q_0(x) \| \tilde{p}_0(x)\right)\\
&\precsim T\left(\mathcal{M}^2+\log \mathcal{M}\right)\!\left(\sqrt{\frac{7n^2}{m}}+\sqrt{\frac{2\log \frac{1}{\delta}}{m}}\right)
+\left(T^4\mathcal{M}^2\log \mathcal{M}\right)\!\left(\sqrt{\frac{7n^2}{m_z^2 m}}+\sqrt{\frac{2\log \frac{1}{\delta}}{m_z^2 m}}\right).
\end{align*}
This concludes the proof.

\end{document}

%% file: shared.tex
\usepackage{lineno} 
\usepackage{lipsum}
\usepackage{graphicx}
\usepackage{epstopdf}
\usepackage{algorithmic}
\usepackage[title]{appendix}
\ifpdf
  \DeclareGraphicsExtensions{.eps,.pdf,.png,.jpg}
\else
  \DeclareGraphicsExtensions{.eps}
\fi

\makeatletter
\newcommand*{\addFileDependency}[1]{
  \typeout{(#1)}
  \@addtofilelist{#1}
  \IfFileExists{#1}{}{\typeout{No file #1.}}
}
\makeatother

\newtheorem*{mydef*}{Definition} 
\newtheorem{mydef}{Definition} 
\newtheorem{theorem}{Theorem}
\newtheorem{corollary}{Corollary}
\newtheorem{lemma}{Lemma}
\newtheorem{assumption}{Assumption}
\newtheorem{prop}{Proposition}
\newtheorem*{prop*}{Proposition}

\newtheorem*{problem*}{Problem}
\newenvironment{remarks}[1][\it{Remarks}]{\textbf{#1. } }

\newcommand{\CA}{\mathcal{A}}

\newcommand{\CL}{\mathcal{L}}

\newcommand{\CW}{\mathcal{W}}
\newcommand{\CN}{\mathcal{N}}

\newcommand{\argmin}{\mathop{\mathrm{argmin}}}
\newcommand{\argmax}{\mathop{\mathrm{argmax}}}
\newcommand{\NN}{\mathcal{NN}}